\documentclass{article}

 \usepackage[preprint]{neurips_2026}

\usepackage[utf8]{inputenc} % allow utf-8 input
\usepackage[T1]{fontenc}    % use 8-bit T1 fonts
\usepackage{hyperref}       % hyperlinks
\usepackage{url}            % simple URL typesetting
\usepackage{booktabs}       % professional-quality tables
\usepackage{amsfonts}       % blackboard math symbols
\usepackage{nicefrac}       % compact symbols for 1/2, etc.
\usepackage{microtype}      % microtypography
\usepackage{xcolor}         % colors
\usepackage{amsmath}
\usepackage{amssymb}
\usepackage{amsthm}
\usepackage{amsmath}
\usepackage{enumitem}
\usepackage{graphicx}
\usepackage{algorithm}
\usepackage{enumitem}
\usepackage{algpseudocode}
\usepackage{dsfont}
\newtheorem{definition}{Definition}
\newtheorem{lemma}{Lemma}
\newtheorem{proposition}{Proposition}
\newtheorem{corollary}{Corollary}
\newtheorem{theorem}{Theorem}

\title{Differentiable Belief-based Opponent Shaping}

\author{%
  Aarav G.~Sane \\
  Department of Computer Science\\
  Purdue University\\
  \texttt{sane0@purdue.edu} \\
  \And
  Karthik Sivachandran \\
  Department of Computer Science \\
  Purdue University \\
  \texttt{ksivacha@purdue.edu} \\
  \AND
  Rohan R.~Paleja  \\
  Department of Computer Science \\
  Purdue University \\
  \texttt{rpaleja@purdue.edu} \\
}

\begin{document}
% \raggedbottom
\setlist{nosep}

\maketitle

% \epigraph{\emph{"Never put passion in front of principle, because even if you win you lose"}}{Mr. Miyagi}

\begin{abstract}
Human coordination often relies on the ability to influence the beliefs of others through strategic action. In multi-agent reinforcement learning, opponent shaping attempts to replicate this influence, though existing methods typically operate within an opponent's parameter, policy, or value space. Meanwhile, belief-manipulation techniques in hidden-role games often rely on hard-coded objectives, such as deception or belief saturation. We propose Differentiable Belief-based Opponent Shaping (D-BOS), a first-order method that treats each observer's belief as the shaped opponent state and differentiates through $k$-step softmax-Bayes belief dynamics. Rather than explicitly rewarding deceptive or cooperative behavior, our method treats the belief state as the target for shaping. This allows the optimal strategy to emerge naturally from the environment's reward structure. This belief-space formulation provides an opponent-shaping signal by differentiating through opponent belief updates, and naturally extends to multiple observers by aggregating gradients over their individual inferred belief trajectories. Empirically, D-BOS outperforms PPO and BBM in hidden-role games, with the largest gains in mixed-motive settings.
\end{abstract}
\vspace{-1.75mm}
\section{Introduction}
\vspace{-1.75mm}
Humans routinely reason about what others believe, and use that reasoning to coordinate, persuade, conceal, or reveal information. This ability is central to Theory of Mind~\citep{premack1978theory,wimmer1983beliefs} and has become an important motivation for computational models of human-agent collaboration and human-robot teaming~\citep{erdogan2022abstracting, paleja2023human}. In collaborative settings, from shared autonomy to ad hoc teaming with heterogeneous strategies~\citep{chen2020joint, paleja2021utility}, the useful partner is not merely the one that chooses good actions in isolation, but the one that understands how its actions change what others think~\citep{shafto2014rational, ho2016showing}. We study this problem in hidden-role partially observable Markov games, such as Avalon~\citep{serrino2019finding}, Rescue-the-General~\citep{bbm}, and the Multi-Agent Coin Game~\citep{raileanu2018modeling}, which are controlled multi-agent testbeds where roles, intentions, and goals are latent, and where an agent's behavior changes the posterior beliefs of allies and adversaries.

Current multi-agent reinforcement learning methods struggle to capture this interaction. Existing opponent-shaping algorithms~\citep{foerster2018lola, aghajohari2024loqa, zhao2022pola, willi2022cola, duque2025aa} anticipate an opponent's learning update to optimize returns, but operate in high-dimensional spaces, require strong assumptions, and are typically restricted to two-player settings. Conversely, belief-manipulation methods like BBM~\citep{bbm} model hidden beliefs but rely on myopic, prescribed objectives like fixed deception. These limitations motivate a natural synthesis: an agent should shape opponent beliefs over multiple steps driven entirely by task reward. Subordinating belief manipulation to task reward ensures it acts as an instrumental mechanism for optimal play rather than a rigid behavioral prior.

We propose \textbf{Differentiable Belief-based Opponent Shaping (D-BOS)}, a first-order method for opponent shaping in belief space. D-BOS treats each observer's (i.e., any other agent in the environment) posterior over hidden roles as the target state to be shaped. It unrolls a differentiable softmax-Bayes update for $k$ steps and backpropagates through the resulting belief trajectory, thereby explicitly optimizing the long-term task value of the induced future belief states. The direction of belief shaping is therefore organically determined: in an adversarial setting, D-BOS can obscure the agent's role, while in a cooperative setting, it can make the correct role easier to infer. We further analyze how D-BOS navigates mixed-motive settings, where it may need to conceal its identity from adversaries while simultaneously revealing it to allies.

Our contributions are:
\begin{itemize}[leftmargin=*,noitemsep]
    \item \textbf{Differentiable belief dynamics for opponent shaping}: D-BOS backpropagates through $k$-step softmax-Bayes belief updates, using the explicit softmax Jacobian to produce temporally extended belief-shaping gradients.
    \item \textbf{A LOLA-style view in belief space}: We show that, under a Bayesian observer model, shaping beliefs is equivalent to shaping the opponent's induced policy, yielding the same meta-gradient structure as LOLA but with a low-dimensional belief state.
    \item \textbf{Error bounds for second-order Theory of Mind}: We bound how approximation error in second-order observations propagates through the $k$-step belief chain and into the shaping gradient, exposing the tradeoff between longer lookahead and belief-model accuracy.
    \item \textbf{Hidden-role experiments}: We evaluate D-BOS against PPO and BBM in Rescue-the-General, Avalon variants, and multi-agent CoinGame to study whether belief-space shaping improves return, belief trajectories, and stability.
\end{itemize}

\vspace{-1.75mm}
\section{Related Work}
\vspace{-1.75mm}
We situate our method at the intersection of three primary lines of research: algorithms for active opponent shaping, computational models of Theory of Mind for belief manipulation, and the embedding of differentiable belief updates within neural architectures.

\paragraph{Opponent Shaping}
Opponent shaping studies how one learner can improve its own long-term outcome by accounting for how its actions change other learners. LOLA~\citep{foerster2018lola} introduced this idea by differentiating through an opponent's anticipated gradient step, which yields a meta-gradient involving mixed second derivatives and requires a differentiable model of the opponent's learning update. Subsequent work refined this basic formulation: SOS~\citep{letcher2019sos} modifies the shaping term to improve stability in differentiable games, COLA~\citep{willi2022cola} addresses consistency when both agents use opponent-aware updates, and POLA~\citep{zhao2022pola} uses a proximal formulation to reduce sensitivity to policy parameterization. M-FOS~\citep{lu2022mfos} instead treats opponent shaping as a meta-learning problem, avoiding a specific differentiable opponent optimizer. Other work reduces or changes the object being differentiated through. LOQA~\citep{aghajohari2024loqa} assumes opponents act through softmax Q-values rather than through an explicit policy-gradient learner. Advantage Alignment~\citep{duque2025aa} shows that LOLA- and LOQA-style shaping terms can be expressed through advantage products, yielding a first-order algorithm. D-BOS is closest in spirit to this line of work: it also seeks a first-order shaping signal. Rather than shaping the advantage function, D-BOS shapes an explicit posterior belief state with a known softmax-Bayes update. This preserves the opponent-shaping viewpoint while moving the differentiable update map into the hidden-role belief dynamics.
\vspace{-1.75mm}
\paragraph{Belief Manipulation and Theory of Mind}
Belief manipulation has been studied most directly in hidden-role and deception settings. Bayesian Belief Manipulation (BBM)~\citep{bbm} is our closest predecessor. In this framework, the agent estimates what an observer believes about its role and computes a Bayes factor to quantify the information revealed by a specific action, as shown in Equation \eqref{eq:bayes_factor}.
\begin{align}
    \label{eq:bayes_factor}
    \rho = \frac{\pi_{z^*}(a | h_{i,j,i})}{\sum_{z \in \mathcal{Z}} \pi_z(a | h_{i,j,i}) P_{i,j}(A_i^z)}
\end{align}
Here, $\rho$ is the Bayes factor, $\pi_{z^*}$ is the likelihood under the shaping agent's true role $z^*$, and $P_{i,j}$ represents the observer's current belief. $h_{i,j,i}$ represents the second-order Theory of Mind estimate of what agent $A_i$ (the shaping agent) thinks agent $A_j$ (the observer) thinks $A_i$ observes. This value is then transformed into an intrinsic reward signal, as expressed by $r_{int} = -\log(\rho)$, incentivizing deception. 
% \begin{align}
% \label{eq:bbm_reward}
% r_{int} = -\log(\rho)
% \end{align}
While BBM provides a practical way to train deceptive agents, its shaping signal is fundamentally \emph{myopic}, as it is computed independently at each timestep and relies on a prescribed direction (deception) rather than the environment's task reward.

More broadly, Theory-of-Mind models have been used to infer agents' latent goals, beliefs, and future behavior. ToMnet~\citep{rabinowitz2018tomnet} learns to model other agents from behavior, and Oguntola~\citep{oguntola2025} formalizes recursive Theory of Mind across several multi-agent testbeds. Similarly, Alon et al. (2023)~\citep{schulz2023} study how recursive reasoning can lead to deception in communication games. These works operate primarily as passive observers or heuristic frameworks. They do not yield a differentiable shaping gradient that can be used for end-to-end reinforcement learning, which precludes them from serving as active opponent-shaping baselines.
Finally, while Model-Based Opponent Modeling (MBOM)~\citep{mobm} uses environment models to adapt to reasoning learners, it is distinct from D-BOS as it focuses on adaptation based on an inferred policy of the opponent. Furthermore, MBOM requires the training of an explicit environment transition model to simulate recursive imagination, while D-BOS is environment-model-free, requiring only the differentiable map of the observer's belief update.
\vspace{-1.75mm}
\paragraph{Differentiable Belief Updates}
Differentiable filtering methods show that Bayesian updates can be embedded inside neural computation. QMDP-net~\citep{karkus2017qmdp} encodes Bayesian filtering and QMDP planning as differentiable neural network layers for partially observable control. IPOMDP-net~\citep{han2019ipomdp} extends this idea to multi-agent settings with interactive belief updates. These methods use differentiable beliefs to help an agent plan under its own uncertainty. D-BOS uses the same computational affordance for a different purpose: differentiating through the belief update of an \emph{observer} to shape what that observer will believe in the future and how the agent's actions affect these belief dynamics. This distinction is important in hidden-role games, where the belief state being optimized is not only the agent's internal uncertainty, but another agent's posterior over the shaper's latent role.

\vspace{-1.75mm}
\section{Preliminaries}
\vspace{-1.75mm}

% \subsection{Hidden-Role Partially Observable Markov Games}
We study $n$-player \emph{partially observable hidden-role games}, a special case of stochastic games~\citep{shapley1953stochastic} and partially observable stochastic games~\citep{hansen2004posg}. We formalize each game as a tuple
$
  \langle \mathcal{N},\, \mathcal{Z},\, \mathcal{S},\, \{\mathcal{O}^i\},\, \mathcal{A},\, T,\, \{r^i\},\, \Omega,\, \gamma \rangle,
$
where $\mathcal{N} = \{1,\ldots,n\}$ is the set of agents; $\mathcal{Z} = \{0,\ldots,|\mathcal{Z}|-1\}$ is a finite set of \emph{hidden role hypotheses} (e.g.\ agent identities or communication slots); $\mathcal{S}$ is the shared world state; $\mathcal{O}^i$ is agent $i$'s private observation space; $\mathcal{A}$ is the joint action space; $T: \mathcal{S} \times \mathcal{A} \to \Delta(\mathcal{S})$ is the transition kernel; $r^i: \mathcal{S} \times \mathcal{A} \to \mathbb{R}$ is agent $i$'s reward; $\Omega: \mathcal{S} \to \prod_i \Delta(\mathcal{O}^i)$ is the observation function; and $\gamma \in (0,1)$ is the discount factor. At the start of each episode, a latent role $z \in \mathcal{Z}$ is sampled and held fixed; agents hold knowledge of their assigned role. All agents share a role-conditioned policy network $\pi_\theta(a \mid o, z)$, which outputs action probabilities given a local observation and a role hypothesis.
% \paragraph{Nomenclature} 
To clarify the interaction dynamics within our framework, we distinguish between two primary entities:

\begin{itemize}[leftmargin=*,noitemsep]
    \item Agent: We refer to the entity that shapes and optimizes the belief-based meta-gradient as the agent.
    \item Observers: All other participants in the environment whose internal beliefs are the target of the shaping process are referred to as observers.
\end{itemize}

% \subsection{Belief States and Role Hypotheses}
In any given environment, an observer maintains a belief distribution over the set of possible role assignments $z \in \mathcal{Z}$. While the specific composition and size of $\mathcal{Z}$ vary according to the rules of the environment, we treat $\mathcal{Z}$ as a generic discrete set for our mathematical formulation.

\subsection{Learning with Opponent-Learning Awareness (LOLA)}
LOLA~\citep{foerster2018lola} studies games in a two-player setting, where an agent (agent 1) accounts for how its own policy parameters $\theta_1$ influence the anticipated update of the observer (agent 2), as shown in Eq. \eqref{eq:lola_objective}:\begin{equation}\label{eq:lola_objective}\Delta \theta_1 = \nabla_{\theta_1} V_1(\theta_1, \theta_2 + \Delta \theta_2)\end{equation}In Equation \eqref{eq:lola_objective}, $V_1$ is the value function of the agent and $\Delta \theta_2$ is the observer's learning step. LOLA optimizes through this parameter update, which typically requires second-order derivatives and knowledge of the observer's learning dynamics.

\subsection{BBM and Second-Order Belief Updates}
BBM~\citep{bbm} adopts a second-order Theory of Mind (ToM) framework where the agent estimates an observer's posterior over the agent's hidden role. Because the agent cannot directly access an observer $j$'s private view of the agent $i$, it must estimate what the observer thinks the agent can see. This second-order observation estimate is denoted as $\hat{o}_{i,j,i}$. The agent's approximation of the observer's belief update is defined in Equation \eqref{eq:approx-bayes-update}.
\begin{equation}\label{eq:approx-bayes-update}\hat{b}_{t+1}^{(j),z} = \frac{\pi_\theta(a_t^i \mid \hat{o}_{i,j,i}, z) \hat{b}_t^{(j),z}}{\sum_{z' \in Z} \pi_\theta(a_t^i \mid \hat{o}_{i,j,i}, z') \hat{b}_t^{(j),z'}}
\end{equation}
In Equation \eqref{eq:approx-bayes-update}, $\hat{b}_{t+1}^{(j),z}$ is the agent's estimate of the probability that observer $j$ assigns to agent $i$'s role being $z$, and $\pi_\theta(a_t \mid \hat{o}_{i,j,i}, z)$ is the role-conditioned action likelihood for role $z$.

\subsection{Observation Proxies}
\label{sec:obs-proxies}
The implementation of $\hat{o}_{i,j,i}$ in Equation \eqref{eq:approx-bayes-update} requires a proxy for the observer's perspective. We utilize two distinct modes depending on the social complexity of the environment:
\begin{itemize}[leftmargin=*,noitemsep]
    \item Canonical Mode: In environments where observers share a fundamentally similar perspective, the agent uses its own observation as a robust heuristic stand-in. This reduces approximation error from additional learned modules and serves as our primary baseline.
    \item Estimated Mode: In mixed-motive settings, an agent must simultaneously deceive adversaries and coordinate with allies, requiring distinct belief trajectories for each observer. To achieve true multi-observer shaping, we deploy a learned backward-prediction module to explicitly infer $\hat{o}_{i,j,i}$ (what observer $j$ believes agent $i$ can see).
\end{itemize}
Because the Estimated Mode is uniquely suited for navigating conflicting observer perspectives, we focus its evaluation on our mixed-motive testbed, Avalon, while relying on the Canonical proxy alongside public state information for the other environments.

\vspace{-1.75mm}
\section{Method}
\label{sec:methods}
\vspace{-1.75mm}

\subsection{Differentiable Belief Shaping of Opponents}
The Bayesian belief update (Eq.~\ref{eq:approx-bayes-update}) involves a ratio: a likelihood multiplied by a prior, divided by a normalizing constant. This is functionally equivalent to the softmax operation applied in log-space. To make the dependence on $\theta$ explicit and differentiable, we first define a log-likelihood vector $\boldsymbol{\ell}_t \in \mathbb{R}^{|\mathcal{Z}|}$, where each element $\ell_t^z = \log \pi_\theta(a_t \mid o_t, z)$. We then rewrite the belief update as expressed in Equation \eqref{eq:diff-bayes}.
\begin{equation}
\label{eq:diff-bayes}
\mathbf{b}_{t+1} = \text{softmax}\!\left(\boldsymbol{\ell}_t + \log \mathbf{b}_t\right)
\end{equation}
In Equation \eqref{eq:diff-bayes}, the softmax operates over the latent role dimension $|\mathcal{Z}|$, ensuring the belief vector remains a valid probability distribution. The Jacobian of this softmax operation provides a clean gradient flow through each update step, as defined as
% in Equation \eqref{eq:softmax-jac}:
% \begin{equation}
% \label{eq:softmax-jac}
$\frac{\partial b^{z'}_{t+1}}{\partial \ell_t^z} = b^{z'}_{t+1}\!\left(\mathds{1}[z'=z] - b^z_{t+1}\right)$.
% \end{equation}
Here $\mathds{1}$ is a function evaluating to $1$ when the target role $z'$ matches the hypothesis role $z$, and $0$ otherwise.

% \paragraph{BeliefCritic.}
Given this differentiable update rule, we train a \emph{BeliefCritic}, which acts as a learned value function of the agent's local state and the joint estimated belief state of all $m$ observers in the environment. We define the multi-observer belief matrix as $\hat{\mathbf{B}}_t \in (\Delta^{|\mathcal{Z}|})^m$, where each row corresponds to a specific observer's tracked posterior over the agent's role. This function is defined in Equation \eqref{eq:belief-critic}.
\begin{equation}
\label{eq:belief-critic}
V_\phi : \mathcal{O}^i \times (\Delta^{|\mathcal{Z}|})^m \to \mathbb{R},\qquad V_\phi(o_t^{(i)}, \hat{\mathbf{B}}_t) \approx \mathbb{E}\!\left[\, \sum_{s \geq t} \gamma^{s-t} r^{\text{ext}}_s \,\Big|\, o_t^{(i)}, \hat{\mathbf{B}}_t, \pi_\theta \right]
\end{equation}
In Equation \eqref{eq:belief-critic}, $V_\phi$ takes the full observer belief matrix $\hat{\mathbf{B}}_t$ as a direct input, learning the task-value of the environment's overall belief configuration for the agent under the current reward structure. By formalizing the domain over $(\Delta^{|\mathcal{Z}|})^m$, the critic naturally scales to environments with multiple distinct observers (such as the Estimated Mode), while supporting scenarios where observers share identical proxy trajectories (such as the Canonical Mode) without altering the underlying architecture. $V_\phi$ is trained alongside $\pi_\theta$ via standard temporal difference (TD) regression against bootstrapped returns.
We formalize the $k$-step D-BOS loss as 
% \eqref{eq:bos-obj}:
% \begin{equation}
% \label{eq:bos-obj}
$\mathcal{L}_{\text{BOS}}^k = - V_\phi\!\left(o_{t+k}^{(i)},\, \mathbf{B}_{t+k}\right)$
% \end{equation}
Here, the multi-observer belief matrix $\hat{\mathbf{B}}_{t+k}$ is obtained by chaining $k$ softmax-Bayes updates for all modeled observers. Minimizing $\mathcal{L}_{\text{BOS}}^k$ drives the policy parameters $\theta$ toward actions whose induced $k$-step joint belief trajectories land in high-value regions. The critic parameters $\phi$ are held fixed when computing the D-BOS gradient; the gradient flows exclusively through $\hat{\mathbf{B}}_{t+k}$'s dependence on $\theta$.

\textbf{Training Algorithm}
Algorithmically, D-BOS is implemented as PPO with an additional belief-shaping correction computed from the on-policy rollout. Algorithm~\ref{alg:bos-training} gives the implementation-independent training loop. In practice, we instantiate the correction either by directly differentiating the $k$-step belief critic objective or by using an equivalent likelihood-coefficient surrogate that detaches the endpoint coefficients for stability; Appendix~\ref{app:coeff-correction} gives the exact implementation details.

\textbf{Multiple observers and observation modes.}
When several agents observe the shaper, D-BOS can theoretically maintain a separate belief trajectory $\mathbf{b}^{j}_{t:t+k}$ for each observer $A_j$, aggregating the shaping loss across observers. However, the capacity to uniquely shape multiple observers depends on the chosen observation proxy (Section~\ref{sec:obs-proxies}). Under the \emph{canonical mode}, the agent uses its own observation $o^{(i)}_t$ for all observers, effectively collapsing the update into a single, generalized hypothetical observer. True multi-observer shaping, where distinct belief trajectories are modeled, requires the \emph{estimated mode}, which provides observer-specific proxies $\hat{o}_{i,j,i}$.

\subsubsection{Gradient Correction (LOLA-Style)}
Rather than adding $\mathcal{L}_{\text{BOS}}^k$ directly to the PPO surrogate loss, D-BOS follows the LOLA paradigm and injects the shaping gradient as a fixed correction. The total shaping gradient is computed \emph{once} from the on-policy rollout trajectory $\pi_{\theta_0}$ prior to any PPO minibatch updates, as shown in Equation \eqref{eq:shaping-grad}:
\begin{equation}
\label{eq:shaping-grad}
g_{\text{BOS}} = \nabla_\theta \frac{\lambda}{T-k+1} \sum_t \mathcal{L}_{\text{BOS}}^k(t) \;\Big|_{\theta = \theta_0}
\end{equation}
Here $\lambda$ controls the scale of the shaping vector and $\theta_0$ denotes the rollout parameters. During optimization, PPO performs $M$ total inner gradient steps (where $M = \text{epochs} \times \text{minibatches}$). At each of these steps, the parameter update is applied as shown in Equation \eqref{eq:full-gradient}:
\begin{equation}
\label{eq:full-gradient}
\theta \leftarrow \theta - \alpha\!\left(\nabla_\theta \mathcal{L}_{\text{PPO}} + \frac{g_{\text{BOS}}}{M}\right)
\end{equation}
Our implementation also supports a likelihood-coefficient form that computes the same endpoint value sensitivity with respect to role-conditioned action likelihoods, detaches these coefficients, and optimizes a stable surrogate policy loss during PPO. Both forms use the same softmax-Bayes chain and the same endpoint critic $V_\phi(o_{t+k},B_{t+k})$; they differ only in how the resulting correction is applied numerically. Further details of our algorithm are provided in Appendix~\ref{app:coeff-correction}.

\subsubsection{Belief as a Sufficient Statistic for Opponent Policy}

In hidden-role games, we can represent a Bayesian-rational opponent conditioning their actions on their posterior belief about the shaping agent's role:
\begin{equation}
    \pi_{\text{opp}}^*(a \mid o_t) = \sum_{z \in Z} b_t^z \cdot \pi_{\text{opp}}^*(a \mid o_t, z)
    \label{eq:opp-policy}
\end{equation}

\begin{proposition}[Belief is a sufficient statistic]
\label{prop:sufficient-statistic}
Let the opponent be Bayesian-rational with policy taking the role-mixture form (Eq.~\ref{eq:opp-policy}), and assume the role-conditioned responses $\pi_{\text{opp}}^*(\cdot \mid o_t, z)$ are independent of the shaper's parameters $\theta_1$. Then the marginal action distribution $\pi_{\text{opp}}^*(\cdot \mid o_t)$ depends on $\theta_1$ only through $b_t$:
\begin{equation}
    {\frac{\partial \pi_{\text{opp}}^*(a \mid o_t)}{\partial \theta_1} = \sum_{z \in Z} \frac{\partial b_t^z}{\partial \theta_1} \cdot \pi_{\text{opp}}^*(a \mid o_t, z)}
\end{equation}
The gradient of the opponent's policy w.r.t.\ the shaper's parameters $\theta_1$ is fully determined by $\partial \mathbf{b}_t / \partial \theta_1$, the D-BOS gradient.
\end{proposition}
A proof is provided in Appendix~\ref{proof:sufficient-statistic}.
% This is the key insight: \textbf{shaping $\mathbf{b}_t$ is equivalent to shaping the opponent's policy.} Rather than operating on the opponent's parameter space $\theta_2$ (as LOLA does), D-BOS operates on the belief simplex $\Delta_{|Z|}$, a compact $|Z|$-dimensional space with a known, differentiable update rule for the softmax.

\subsubsection{A Meta-Gradient View (and Relation to LOLA)}

Opponent shaping methods can be seen as differentiating through an \emph{opponent update map}.
To make this precise, we introduce an opponent internal state $\eta_t$ that parameterizes their behavior (e.g., network weights, beliefs, or a sufficient statistic).

\begin{theorem}[Opponent shaping via differentiable opponent state]
\label{thm:meta-shaping}
Let the opponent act according to a policy $\pi_{\text{opp}}(\cdot \mid o_t, \eta_t)$ and update an internal state $\eta_t \in \mathbb{R}^d$ via a differentiable map
\begin{equation}
    \eta_{t+1} = U(\eta_t, \theta_1, \xi_t),
    \label{eq:opp-state-update}
\end{equation}
where $\theta_1$ are the shaper's parameters and $\xi_t$ collects variables observed by the opponent (including the shaper's action). For any objective $J_1(\theta_1, \eta_t)$ differentiable in both arguments,
\begin{equation}
    \nabla_{\theta_1} J_1(\theta_1, \eta_t)
    = \underbrace{\frac{\partial J_1}{\partial \theta_1}\bigg|_{\eta_t}}_{\text{direct term}}
    + \underbrace{\left(\nabla_{\theta_1} \eta_t\right)^\top \nabla_{\eta_t} J_1(\theta_1, \eta_t)}_{\text{shaping term}}.
    \label{eq:meta-gradient-decomp}
\end{equation}
Moreover, $\nabla_{\theta_1} \eta_t$ is obtained by differentiating through the recursion in Eq.~\ref{eq:opp-state-update}.
\end{theorem}
A proof is provided in Appendix~\ref{proof:meta-shaping}.

\begin{corollary}[D-BOS as belief-state shaping]
\label{cor:bos-belief-shaping}
Under a Bayesian-rational observer model (Eq.~\ref{eq:opp-policy}), we may take $\eta_t = \mathbf{b}_t$ and $U$ to be the Bayes/softmax update (Eq.~\ref{eq:diff-bayes}). Here, the shaping term in Eq.~\ref{eq:meta-gradient-decomp} depends on $\nabla_{\theta_1}\mathbf{b}_t$, which D-BOS computes by backpropagating through the $k$-step belief chain.
\end{corollary}

\begin{corollary}[LOLA as parameter-state shaping]
\label{cor:lola-state}
LOLA is recovered as another instantiation of Theorem~\ref{thm:meta-shaping} by taking $\eta_t = \theta_2$ (opponent parameters) and the update map
\begin{equation}
    U(\theta_2, \theta_1, \cdot) = \theta_2 + \alpha\, \nabla_{\theta_2} J_2(\theta_1, \theta_2).
\end{equation}
Differentiating through this $U$ introduces the mixed Hessian $\partial^2 J_2 / \partial \theta_1 \partial \theta_2$ that characterizes LOLA.
\end{corollary}

This framing shows that \textbf{D-BOS and LOLA share the same meta-gradient structure}, but D-BOS chooses an opponent state (belief) with a known, low-dimensional, closed-form update map.

\subsubsection{Advantages of Belief-Space Shaping}

\textbf{Dimensionality.} $\mathbf{b}_t \in \Delta^{|Z|}$ has dimension $|Z| - 1$, which is typically far smaller than the parameter space of a deep opponent policy. This compact state makes it tractable to maintain and shape separate belief trajectories for multiple observers at the same time.

\textbf{First-order belief gradients.} The softmax-Bayes formulation in Eq.~\ref{eq:diff-bayes} makes the observer belief update differentiable with respect to the shaper's policy parameters through the likelihood of the shaper's observed actions. D-BOS therefore obtains a first-order shaping signal through the belief update itself, rather than differentiating through an opponent optimizer.

\textbf{Different assumptions.} D-BOS models the state of the observer's posterior rather than the opponent's optimizer state. The opponent need not literally execute Bayes' rule; the belief is a differentiable surrogate maintained by the shaper to model and influence what the observer is likely to infer.

% \subsubsection{Connection to Differentiable Filters}

% QMDP-Net \citep{karkus2017qmdp} and IPOMDP-Net \citep{han2019ipomdp} embed differentiable Bayesian filters into neural architectures for planning under partial observability. D-BOS applies this idea to \emph{opponent modeling}: the differentiable filter tracks the opponent's belief, and its gradient is used to shape that belief. The key distinction is the direction: those methods use differentiable beliefs to \emph{plan for the agent itself}; D-BOS uses differentiable beliefs to \emph{shape what the opponent believes about the agent}.

\vspace{-1.75mm}
\section{Theoretical Analysis: Gradient Error from ToM Approximation}
\vspace{-1.75mm}
\label{sec:theory}

The D-BOS gradient is computed through the shaper's model of what each observer believes. This model may use a second-order estimate $\hat{o}_{i,j,i}$ that differs from the observation proxy $o_{j,i}$ used by the observer's ideal Bayesian update. We bound how this approximation error propagates through the $k$-step belief chain and into the shaping gradient.

\subsection{Per-Step Belief Error}

\begin{definition}[Log-likelihood error]
\label{def:log-likelihood-error}
For observer $A_j$ at step $t$, define the per-step log-likelihood error as:
\begin{equation}
    \varepsilon^{(j)}_t = \max_{z \in Z} \left| \log \pi_\theta(a_t \mid o_{j,i}, z) - \log \pi_\theta(a_t \mid \hat{o}_{i,j,i}, z) \right|
    \label{eq:log-likelihood-error}
\end{equation}
This measures the worst-case discrepancy over roles between the log-likelihoods computed with the observer's true estimate versus the shaper's second-order estimate.
\end{definition}

\begin{lemma}[Softmax Lipschitz property]
\label{lem:softmax-lipschitz}
For any two input vectors $x, y \in \mathbb{R}^{|Z|}$:
\begin{equation}
    \left\| \mathrm{softmax}(x) - \mathrm{softmax}(y) \right\|_1 \leq 2 \left\| x - y \right\|_\infty
\end{equation}
\end{lemma}
A proof is provided in Appendix~\ref{proof:softmax-lipschitz}.

\begin{proposition}[$k$-step belief error]
\label{prop:belief-error}
Let $\mathbf{b}_{t+k}^{(j)}$ be the observer's true belief after $k$ steps and $\hat{\mathbf{b}}_{t+k}^{(j)}$ be the shaper's approximation using $\hat{o}_{i,j,i}$ at each step. Assume the initial beliefs agree: $\hat{\mathbf{b}}_t^{(j)} = \mathbf{b}_t^{(j)}$. Then:
\begin{equation}
    \left\| \hat{\mathbf{b}}_{t+k}^{(j)} - \mathbf{b}_{t+k}^{(j)} \right\|_1 \leq 2 \sum_{s=0}^{k-1} \varepsilon_{t+s}^{(j)}
\end{equation}
If the per-step error is bounded uniformly by $\varepsilon^{(j)} = \max_t \varepsilon_t^{(j)}$:
\begin{equation}
    \left\| \hat{\mathbf{b}}_{t+k}^{(j)} - \mathbf{b}_{t+k}^{(j)} \right\|_1 \leq 2k\varepsilon^{(j)}
\end{equation}
\end{proposition}
A proof is provided in Appendix~\ref{proof:k-step-belief-error}.

\subsection{Gradient Error Bound}

\begin{theorem}[Gradient error from ToM approximation]
\label{thm:grad-error}
Assume: (1) the per-step log-likelihood error is bounded by $\varepsilon^{(j)}$, 
(2) the policy gradient is Lipschitz in observations with constant $L_\pi$, 
(3) beliefs are bounded away from zero: $b^{(j),z}_{t+s} \geq b_{\min} > 0$ 
and $\hat b^{(j),z}_{t+s} \geq b_{\min} > 0$ for all $z, s$, (4) the 
future-state value function $V_\phi(o^{(i)}_{t+k}, \cdot)$ is $L_V$-Lipschitz 
in its belief argument, and (5) the gradient $\nabla_b V_\phi(o^{(i)}_{t+k}, 
\cdot)$ is $L_g$-Lipschitz in its belief argument (in $\|\cdot\|_1$). Then the 
gradient error over a $k$-step window satisfies:
\begin{equation}
\left\| \nabla_\theta L^{(j)}_{\mathrm{appr.}} - \nabla_\theta L^{(j)}_{\mathrm{true}} \right\|_2 
\leq \frac{C(\mathcal{Z}) \cdot k \cdot \max(L_V, L_g)}{b_{\min}^2} 
\cdot \left( k \varepsilon^{(j)} G_\pi + L_\pi \cdot \sup_t \|o_{j,i} - \hat o_{i,j,i}\|_\infty \right),
\end{equation}
where $G_\pi = \max_{t, z} \|\nabla_\theta \log \pi_\theta(a_t \mid 
\hat o_{i,j,i}, z)\|_2$ is the maximum per-role policy-gradient norm and $C(\mathcal{Z}) = (22|\mathcal{Z}| - 18)$
\end{theorem}

A complete proof is provided in Appendix~\ref{app:thm2-proof}.

\subsection{Implications}

The bound highlights three practical tradeoffs:

\textbf{Window size $k$ vs.\ ToM accuracy.} The gradient error grows with the length of the belief rollout. Longer lookahead windows can expose useful multi-step belief trajectories, but they also amplify errors in the shaper's second-order observation model. This predicts the empirical need to tune $k$ rather than always choosing the longest available horizon.

\textbf{Number of observers.} Observers for which the shaper has poor ToM estimates (large $\varepsilon^{(j)}$) contribute disproportionately to the gradient error. This suggests observer weighting as a natural extension, e.g., allocating more shaping weight to observers whose beliefs can be modeled reliably.

\textbf{Belief concentration.} The $1/b_{\min}^2$ factor means the gradient error is amplified when the observer's belief is nearly certain about some role. This is intuitive: when the observer is very confident, a small error in the estimated belief corresponds to a large error in the log-belief gradient. In practice, this suggests that the D-BOS gradient is most reliable when the observer is still uncertain, before the belief state has saturated.

\textbf{Comparison to BBM.} BBM also uses a second-order observer estimate, but only to construct a single-step intrinsic reward. D-BOS trades this myopic reward for a differentiable $k$-step belief trajectory. The benefit is temporal credit assignment over future beliefs; the cost is that approximation error can accumulate through the unrolled belief updates. The bound makes this tradeoff explicit.

\vspace{-1.75mm}
\section{Experiments and Results}
\label{sec:experiments}
\vspace{-1.75mm}

We evaluate D-BOS across three environments of varying social complexity: Rescue-the-General, Avalon, and Multi-Agent Coin Game. Our experiments aim to answer three questions. First, can differentiating through belief dynamics produce a nonzero shaping signal that is meaningfully coupled to future return? Second, does increasing the belief-planning horizon $k$ improve performance in controlled mixed-motive settings? Third, how stable is the shaping signal in richer hidden-role games with several observers and nonstationary co-trained policies?
\vspace{-1.75mm}
\paragraph{Baselines}
\begin{enumerate}[leftmargin=*,noitemsep]
    \item \textbf{PPO / no shaping}: role-conditioned PPO trained on extrinsic reward only.
    \item \textbf{BBM} \citep{bbm}: PPO with the per-timestep Bayes-factor intrinsic reward.
    \item \textbf{D-BOS (ours)}: $k$-step differentiable belief-based opponent shaping with different $k$
\end{enumerate}

We omit parameter-space opponent-shaping and meta-learning baselines, such as LOLA, Advantage Alignment~\citep{duque2025aa}, and M-FOS~\citep{lu2022mfos} due to computational intractability. In our architecture, the D-BOS correction cleanly updates the low-dimensional role-conditioned policy head, leaving the shared visual encoder unchanged. Conversely, parameter-space methods require differentiating through an opponent's learning step over shared parameters, introducing prohibitive cross-role mixed second derivatives. Even first-order or meta-learning variants operate over high-dimensional spaces, rendering them computationally infeasible in visually complex grid worlds and incompatible with the exact, head-only decomposition provided by explicit belief-space shaping. For these reasons we stick to BBM and PPO as our baselines.
\vspace{-1.75mm}
\paragraph{Metrics}
For each environment, we report the shaping agent's score or win rate when available. The specific role hypotheses $\mathcal{Z}$ vary by setting:
\begin{itemize}[leftmargin=*,noitemsep]
    \item \textbf{Avalon}: A 5-player social deduction game where spies (team of 2) use social actions to deceive resistance members (team of 3) and sabotage missions.
    \item \textbf{Rescue-the-General (RTG)}: A 6-agent visual grid world where a red team aims to eliminate a general before a blue team can successfully rescue them.
    \item \textbf{Multi-agent CoinGame}: A spatial game where a shaper with a randomized role (altruistic vs.\ selfish) competes against three observers to collect coins.
\end{itemize}

We provide specific environment details in Appendix \ref{app:env-details}
\vspace{-1.75mm}
\subsection{Results}
\vspace{-1.75mm}
\label{sec:results}

Tables~\ref{tab:rtg-final}, \ref{tab:coin-game-final}, and \ref{tab:avalon-final} summarize the performance of D-BOS against the PPO and BBM baselines across all three environments.
A consistent finding across environments is that BBM actively underperforms vanilla PPO in both Avalon (win rate $0.243$ vs.\ $0.343$) and RTG (win rate $0.047$ vs.\ $0.101$). This empirically validates a core motivation of D-BOS: hardcoded, myopic deception objectives are fundamentally brittle. An agent that greedily minimizes a Bayes factor at every timestep behaves with such artificial caution that observers easily detect it through second-order reasoning. D-BOS avoids this pathology by letting the environment's task reward organically dictate the direction and magnitude of belief shaping over a multi-step horizon.

% --- Combined Table Environment (Single Row, Top Aligned) ---
\begin{table}[htbp] % Use \begin{table*} if your document is two-column
    \centering
    \small
    
    % --- Table 1: RTG ---
    \begin{minipage}[t]{0.32\textwidth}
        \vspace{0pt} % <--- Forces top alignment
        \centering
        \resizebox{\linewidth}{!}{
        \begin{tabular}{lc}
            \toprule
            Method & Win Rate \\
            \midrule
            PPO (No shaping) & 0.101 $\pm$ 0.042 \\
            BBM & 0.047 $\pm$ 0.011 \\
            D-BOS, $k=1$ & 0.185 $\pm$ 0.073 \\
            D-BOS, $k=3$ & \textbf{0.205 $\pm$ 0.138} \\
            D-BOS, $k=5$ & 0.066 $\pm$ 0.016 \\
            \bottomrule
        \end{tabular}
        }
        \vspace{3pt}
        \caption{Rescue-the-General (RTG) Evaluation Win Rates for the Red agent. We report the mean $\pm$ standard error at the end of co-training}
        \label{tab:rtg-final}
    \end{minipage}
    \hfill
    % --- Table 2: Coin Game ---
    \begin{minipage}[t]{0.32\textwidth}
        \vspace{0pt} % <--- Forces top alignment
        \centering
        \resizebox{\linewidth}{!}{
        \begin{tabular}{lc}
            \toprule
            Method & Coin Return \\
            \midrule
            PPO (No shaping) & -4.960 $\pm$ 0.147 \\
            BBM & -5.307 $\pm$ 0.285 \\
            D-BOS, $k=1$ & -5.000 $\pm$ 0.251 \\
            D-BOS, $k=3$ & -5.013 $\pm$ 0.195 \\
            D-BOS, $k=5$ & \textbf{-4.227 $\pm$ 0.848} \\
            \bottomrule
        \end{tabular}
        }
        \vspace{3pt}
        \caption{Multi-Agent Coin Game Returns for the red agent (the shaper agent). Higher coin return is better. We report the mean $\pm$ standard error at the end of co-training and evaluating against frozen policies}
        \label{tab:coin-game-final}
    \end{minipage}
    \hfill
    % --- Table 3: Avalon ---
    \begin{minipage}[t]{0.32\textwidth}
        \vspace{0pt} % <--- Forces top alignment
        \centering
        \resizebox{\linewidth}{!}{
        \begin{tabular}{lc}
            \toprule
            Method & Win Rate \\
            \midrule
            PPO (No shaping) & 0.343 $\pm$ 0.032 \\
            BBM & 0.243 $\pm$ 0.098 \\
            D-BOS, $k=1$ & 0.363 $\pm$ 0.110 \\
            D-BOS, $k=3$ & 0.423 $\pm$ 0.088 \\
            D-BOS, $k=5$ & 0.413 $\pm$ 0.039 \\
            \midrule
            D-BOS, $k=1$ (Est.) & \textbf{0.557 $\pm$ 0.048} \\
            D-BOS, $k=3$ (Est.) & 0.507 $\pm$ 0.012 \\
            \bottomrule
        \end{tabular}
        }
        \vspace{3pt}
        \caption{Avalon Evaluation Win Rates. We report the mean $\pm$ standard error at the end of co-training and evaluating against frozen policies}
        \label{tab:avalon-final}
    \end{minipage}
\end{table}

\paragraph{Avalon: Social Deduction with Mixed-Motive Shaping.}
Avalon is our primary testbed; its information asymmetry requires the agent to shape multiple competing observer beliefs using purely social actions. As Table~\ref{tab:avalon-final} shows, canonical D-BOS ($k \ge 3$) strongly outperforms PPO and BBM, proving the necessity of multi-step belief planning in social deduction. Crucially, because Avalon is a mixed-motive setting, relying on a shared Canonical proxy limits the agent's ability to model distinct perspectives. Using the \textbf{Estimated Mode} to actively model these conflicting viewpoints yields our strongest overall result. By minimizing belief-update error via observer-specific proxies, the agent generates accurate shaping gradients that simultaneously deceive adversaries and coordinate with allies, thus performing well in a mixed-motive setting.

\paragraph{Rescue-the-General (RTG): High-Dimensional Scaling.} 
RTG introduces severe visual complexity, which inherently inflates the Theory of Mind approximation error ($\varepsilon$). As Theorem~\ref{thm:grad-error} predicts, maintaining a stable gradient here mathematically requires shorter belief horizons. Table~\ref{tab:rtg-final} confirms this dynamic: D-BOS peaks at $k=3$ but collapses at $k=5$ as compounding errors overwhelm the shaping signal. While visual deception natively suffers from high variance~\citep{bbm} and both PPO and BBM policies collapse during training, D-BOS establishes a robust performance trend. Full training curves in Appendix~\ref{app:diagnostic-figures} explicitly visualize the late-stage bifurcation predicted by our belief-saturation bounds. This motivates more balanced hyper-parameter tuning but D-BOS shows potential even in highly complex games.

\paragraph{CoinGame: Boundary Conditions.} 
CoinGame tests a strict algorithmic limitation: latent roles resample ephemerally and signaling is purely spatial. Here, D-BOS achieves only modest gains over PPO. This cleanly delineates the method's scope: multi-step belief shaping excels with persistent identities and rich communication, but provides limited advantage when roles constantly reset and belief channels carry sparse information.

\vspace{-1.75mm}
\section{Discussion and Limitations}
\label{sec:discussions_limitations}
\vspace{-1.75mm}

BOS is best viewed as an explicit mechanism for belief-space opponent shaping, not as a drop-in reward bonus that automatically dominates PPO. The experiments show that the mechanism is active: the belief critic can learn future belief-state value, the softmax-Bayes chain produces differentiable gradients, and these gradients are injected into the policy update. They also show the main optimization challenge. In complex hidden-role games, the shaping signal can be small or noisy relative to the policy gradient. Increasing $\lambda$ or adaptively amplifying the gradient makes the effect visible, but can hurt performance if the critic or belief trajectory introduces error.

\textbf{Broader impacts and interpretability.} Beyond the specific hidden-role games studied here, D-BOS addresses a fundamental challenge in multi-agent systems: how an agent should systematically influence the mental models others form about it. This dynamic arises whenever agents with private information must coordinate or compete. Practical examples include autonomous vehicles signaling intent to human drivers, assistive robots communicating capabilities to human teammates, and negotiation algorithms managing impressions in strategic interactions. By grounding opponent shaping in an explicit, low-dimensional belief simplex rather than opaque parameter perturbations, D-BOS provides an auditable, realistic mechanism. Further, practitioners can directly inspect what the agent is attempting to make observers believe by analyzing the critic's value assignments over the belief space. This transparency is a strict requirement for deploying influence-aware agents in domains where accountability is necessary.

\textbf{Evaluation against co-trained opponents.} Because D-BOS shapes the beliefs of opponents that are simultaneously being trained, the resulting policies are adapted to each other. The win rates reflect co-trained evaluation, and transferability to unseen opponents remains an open question. Evaluating D-BOS against held-out or independently trained opponents would provide a stronger test of whether the learned shaping strategies generalize beyond the co-training distribution.

\textbf{Future directions.} Extending D-BOS to continuous belief spaces (\ref{app:cont_role_spaces}), analyzing performance with $n$-order Theory of Mind, and developing adaptive observer-weighting schemes (as suggested by the per-observer error bound in Theorem~\ref{thm:grad-error}) are natural next steps. We also notice that recursive modeling reduces $\varepsilon$ but introduces additional learned modules whose own approximation error must be controlled, motivating a joint search over the horizon $k$ and reasoning depth $n$.

\vspace{-1.75mm}
\section{Conclusion}
\vspace{-1.75mm}
We introduced Belief-based Opponent Shaping, an opponent-shaping method that treats opponent belief states as the differentiable opponent state. D-BOS backpropagates through $k$ steps of softmax-Bayes belief dynamics, allowing the reward function rather than a hand-coded deception objective to determine whether beliefs should move toward cooperation, concealment, or adversarial confusion. Our theory connects the method to meta-gradient opponent shaping and bounds the cost of approximate second-order Theory of Mind. Empirically, D-BOS produces real shaping gradients and can improve controlled mixed-motive behavior, while richer hidden-role games reveal the sensitivity of belief-space shaping to critic quality and gradient scale. We see this as a first step toward opponent shaping methods that operate on the beliefs agents form about one another, rather than directly on policy parameters.

% \begin{ack}
% Use unnumbered first level headings for the acknowledgments. All acknowledgments

% Do {\bf not} include this section in the anonymized submission, only in the final paper. You can use the \texttt{ack} environment provided in the style file to automatically hide this section in the anonymized submission.
% \end{ack}

\newpage
\bibliographystyle{unsrt}
\bibliography{references}

@inproceedings{aghajohari2024loqa,
    author = {Milad Aghajohari and
Juan Agustin Duque and
Tim Cooijmans and
Aaron Courville},
    booktitle = {International Conference on Learning Representations,
{ICLR} 2024},
    title = {{LOQA:} Learning with Opponent Q-Learning Awareness},
    year = {2024}
}

@inproceedings{bbm,
    author = {Aitchison, Matthew and Benke, Lyndon and Sweetser, Penny},
    booktitle = {First International Workshop on Deceptive AI},
    series = {Communications in Computer and Information Science},
    title = {Learning to Deceive in Multi-Agent Hidden Role Games},
    year = {2021}
}

@inproceedings{chen2020joint,
    author = {Chen, Letian and Paleja, Rohan and Ghuy, Muyleng and Gombolay, Matthew},
    booktitle = {Proceedings of the 2020 ACM/IEEE International Conference on Human-Robot Interaction (HRI)},
    title = {Joint Goal and Strategy Inference across Heterogeneous Demonstrators via Reward Network Distillation},
    year = {2020}
}

@inproceedings{duque2025aa,
    author = {Duque, Juan Agustin and Aghajohari, Milad and Cooijmans, Tim and Ciuca, Razvan and Zhang, Tianyu and Gidel, Gauthier and Courville, Aaron},
    booktitle = {International Conference on Learning Representations (ICLR)},
    note = {Oral},
    title = {Advantage Alignment Algorithms},
    year = {2025}
}

@inproceedings{erdogan2022abstracting,
    author = {Erdogan, Emre and Dignum, Frank and Verbrugge, Rineke and Yolum, Pinar},
    booktitle = {HHAI2022: Augmenting Human Intellect},
    title = {Abstracting Minds: Computational Theory of Mind for Human-Agent Collaboration},
    year = {2022}
}

@inproceedings{foerster2018lola,
    author = {Foerster, Jakob N and Chen, Richard Y and Al-Shedivat, Maruan and Whiteson, Shimon and Abbeel, Pieter and Mordatch, Igor},
    booktitle = {Proceedings of the 17th International Conference on Autonomous Agents and MultiAgent Systems (AAMAS)},
    title = {Learning with Opponent-Learning Awareness},
    year = {2018}
}

@inproceedings{han2019ipomdp,
    author = {Yanlin Han and
Piotr J. Gmytrasiewicz},
   booktitle = {Proceedings of the Thirty-Third AAAI Conference on Artificial Intelligence and Thirty-First Innovative Applications of Artificial Intelligence Conference and Ninth AAAI Symposium on Educational Advances in Artificial Intelligence},
    title = {IPOMDP-Net: {A} Deep Neural Network for Partially Observable Multi-Agent
Planning Using Interactive POMDPs},
    year = {2019}
}

@inproceedings{hansen2004posg,
    author = {Hansen, Eric A. and Bernstein, Daniel S. and Zilberstein, Shlomo},
    booktitle = {AAAI Workshop - Technical Report},
    title = {Dynamic Programming for Partially Observable Stochastic Games},
    year = {2004}
}

@inproceedings{ho2016showing,
    author = {Mark K. Ho and
Michael L. Littman and
James MacGlashan and
Fiery Cushman and
Joseph L. Austerweil},
    booktitle = {Annual Conference
on Neural Information Processing Systems 2016},
    title = {Showing versus doing: Teaching by demonstration},
    year = {2016}
}

@inproceedings{karkus2017qmdp,
    author = {P{\'{e}}ter Karkus and
David Hsu and
Wee Sun Lee},
    booktitle = {Annual Conference on Neural Information Processing Systems 2017},
    title = {QMDP-Net: Deep Learning for Planning under Partial Observability},
    year = {2017}
}

@inproceedings{letcher2019sos,
    author = {Alistair Letcher and
Jakob N. Foerster and
David Balduzzi and
Tim Rockt{\"{a}}schel and
Shimon Whiteson},
    booktitle = {7th International Conference on Learning Representations, {ICLR} 2019},
    title = {Stable Opponent Shaping in Differentiable Games},
    year = {2019}
}

@inproceedings{lu2022mfos,
    author = {Christopher Lu and
Timon Willi and
Christian A. Schr{\"{o}}der de Witt and
Jakob N. Foerster},
    booktitle = {International Conference on Machine Learning, {ICML} 2022},
    title = {Model-Free Opponent Shaping},
    year = {2022}
}

@inproceedings{mobm,
    author = {Xiaopeng Yu and
Jiechuan Jiang and
Wanpeng Zhang and
Haobin Jiang and
Zongqing Lu},
    booktitle = {Annual Conference
on Neural Information Processing Systems 2022, NeurIPS 2022},
    title = {Model-Based Opponent Modeling},
    year = {2022}
}

@phdthesis{oguntola2025,
    author = {Oguntola, Ini},
    school = {Carnegie Mellon University},
    title = {Theory of Mind in Multi-Agent Systems},
    year = {2025}
}

@inproceedings{paleja2021utility,
      title={The Utility of Explainable AI in Ad Hoc Human-Machine Teaming}, 
      author={Rohan Paleja and Muyleng Ghuy and Nadun Ranawaka Arachchige and Reed Jensen and Matthew Gombolay},
      year={2022},
      eprint={2209.03943},
      archivePrefix={arXiv},
      primaryClass={cs.AI},
      url={https://arxiv.org/abs/2209.03943}, 
}

@article{paleja2023human,
    author = {Paleja, Rohan and Silva, Andrew and Chen, Letian and Gombolay, Matthew},
    journal = {Current Robotics Reports},
    title = {Human-Robot Teaming: Grand Challenges},
    year = {2023}
}

@article{premack1978theory,
    author = {Premack, David and Woodruff, Guy},
    journal = {Behavioral and Brain Sciences},
    title = {Does the Chimpanzee Have a Theory of Mind?},
    year = {1978}
}

@inproceedings{rabinowitz2018tomnet,
    author = {Neil C. Rabinowitz and
Frank Perbet and
H. Francis Song and
Chiyuan Zhang and
S. M. Ali Eslami and
Matthew Botvinick},
    booktitle = {Proceedings of the 35th International Conference on Machine Learning,
{ICML} 2018},
    series = {Proceedings of Machine Learning Research},
    title = {Machine Theory of Mind},
    year = {2018}
}

@inproceedings{raileanu2018modeling,
    author = {Roberta Raileanu and
Emily Denton and
Arthur Szlam and
Rob Fergus},
    booktitle = {Proceedings of the 35th International Conference on Machine Learning,
{ICML} 2018},
    series = {Proceedings of Machine Learning Research},
    title = {Modeling Others using Oneself in Multi-Agent Reinforcement Learning},
    year = {2018}
}

@article{schulz2023,
    author = {Alon, Nitay and Schulz, Lion and Rosenschein, Jeffrey S. and Dayan, Peter},
    journal = {Open Mind},
    title = {A (Dis-)information Theory of Revealed and Unrevealed Preferences: Emerging Deception and Skepticism via Theory of Mind},
    year = {2023}
}

@inproceedings{serrino2019finding,
    author = {Jack Serrino and
Max Kleiman{-}Weiner and
David C. Parkes and
Josh Tenenbaum},
    booktitle = {Advances in Neural Information Processing Systems 32: Annual Conference
on Neural Information Processing Systems 2019, NeurIPS 2019, December
8-14, 2019, Vancouver, BC, Canada},
    title = {Finding Friend and Foe in Multi-Agent Games},
    year = {2019}
}

@article{shafto2014rational,
    author = {Shafto, Patrick and Goodman, Noah D and Griffiths, Thomas L},
    journal = {Cognitive Psychology},
    title = {A Rational Account of Pedagogical Reasoning: Teaching by, and Learning from, Examples},
    year = {2014}
}

@article{shapley1953stochastic,
    author = {Shapley, Lloyd S.},
    journal = {Proceedings of the National Academy of Sciences},
    title = {Stochastic Games},
    year = {1953}
}

@inproceedings{willi2022cola,
    author = {Timon Willi and
Alistair Letcher and
Johannes Treutlein and
Jakob N. Foerster},
    booktitle = {International Conference on Machine Learning, {ICML} 2022},
    title = {{COLA:} Consistent Learning with Opponent-Learning Awareness},
    year = {2022}
}

@article{wimmer1983beliefs,
    author = {Wimmer, Heinz and Perner, Josef},
    journal = {Cognition},
    number = {1},
    title = {Beliefs About Beliefs: Representation and Constraining Function of Wrong Beliefs in Young Children's Understanding of Deception},
    year = {1983}
}

@inproceedings{zhao2022pola,
    author = {Stephen Zhao and
Chris Lu and
Roger B. Grosse and
Jakob N. Foerster},
    booktitle = {Neural Information Processing Systems 2022},
    title = {Proximal Learning With Opponent-Learning Awareness},
    year = {2022}
}

%%%%%%%%%%%%%%%%%%%%%%%%%%%%%%%%%%%%%%%%%%%%%%%%%%%%%%%%%%%%
\newpage
\appendix

\section{Technical appendices and supplementary material}
% Technical appendices with additional results, figures, graphs, and proofs may be submitted with the paper submission before the full submission deadline (see above). You can upload a ZIP file for videos or code, but do not upload a separate PDF file for the appendix. There is no page limit for the technical appendices. 

% Note: Think of the appendix as ``optional reading'' for reviewers. The paper must be able to stand alone without the appendix; for example, adding critical experiments that support the main claims to an appendix is inappropriate. 

\subsection{Continuous Role Spaces}
\label{app:cont_role_spaces}
A natural question is whether our method only applies to discrete role spaces.
While we focus on a discrete role set $Z$, the opponent-state view in Theorem~\ref{thm:meta-shaping} suggests a generalization to \emph{continuous} latent roles $z \in \mathcal{Z} \subset \mathbb{R}^m$. In that setting, the observer maintains a belief density $b_t(z)$ and performs a Bayes update:
\begin{equation}
    b_{t+1}(z) \propto \pi_\theta(a_t \mid o_t, z)\, b_t(z),
\end{equation}
This remains a differentiable opponent update map. Practically, computing the normalizer requires integration over $\mathcal{Z}$, so one would represent $b_t$ with an approximate family (e.g., particles, Gaussian/mixture variational beliefs, or a normalizing-flow posterior) and differentiate through the resulting approximate filter. This would preserve the key benefit of belief-space shaping, i.e., differentiating through a known state update map, while trading exactness for approximate inference in continuous spaces. This direction is complementary to first-order LOLA approximations such as Advantage Alignment \citep{duque2025aa}, which remove Hessians in \emph{parameter}-space shaping; here, the goal is to shift the shaped opponent state from parameters $\theta_2$ to an explicit latent belief representation.

\subsection{Environment Details}
\label{app:env-details}

For fair comparison across methods, we use the same architectural backbone within each environment. For Rescue-the-General, we use the default CNN policy path, whereas for Avalon and CoinGame, we use an MLP policy path. The rollout length and total training budget vary by environment.

\paragraph{Rescue-the-General (RTG).}
RTG is a partially observable grid-world hidden-role game with six agents: one red agent, one green agent, and four blue agents.  The main evaluation scenario uses a $32 \times 32$ map with local RGB observations.  With the default view radius used in the paper runs, each agent receives a $3 \times 45 \times 45$ uint8 egocentric crop containing terrain, visible agents, projectiles, bodies, vote information, and task objects.  Roles are not globally revealed; visibility depends on the scenario's role-visibility setting and local observation.

The discrete action set has ten actions:
\begin{equation}
\{\textsc{noop},\textsc{up},\textsc{down},\textsc{left},\textsc{right},\textsc{act},
\textsc{shoot-up},\textsc{shoot-down},\textsc{shoot-left},\textsc{shoot-right}\}.
\end{equation}
Movement changes the agent's grid location subject to walls and map constraints.  Shooting actions damage the first target in the selected direction subject to range and cooldown.  The \textsc{act} action is context dependent: it can interact with the general or map objects, and it can initiate a vote when the agent is near a dead body or vote button.  During a vote, actions are interpreted as vote choices rather than movement/shooting commands.

The red team's objective is to find and kill or prevent rescue of the general; the blue team attempts to rescue the general; the green agent is an additional role-conditioned participant in the RTG benchmark dynamics.  Terminal outcomes are logged with RTG-compatible strings: \texttt{general\_killed} corresponds to a red win and \texttt{general\_rescued}/\texttt{blue\_win} corresponds to a blue win.  We train RTG with the CNN policy path, role-conditioned policy/value heads, and a deception/backward-prediction module when estimated-observation ablations are enabled.  The main RTG runs use 64 parallel environments, rollout length 16, and one million environment interactions per logged epoch.

\paragraph{Avalon5.}
Avalon5 is a five-player Resistance/Avalon-style social deduction environment.  It has two spies and three resistance players.  
% The mission team sizes follow the standard five-player Resistance table:
% \begin{equation}
% 2 \rightarrow 3 \rightarrow 2 \rightarrow 3 \rightarrow 3.
% \end{equation}
The game alternates through three phases: \textsc{propose}, \textsc{vote}, and \textsc{quest}.  In the proposal phase, the current leader selects a mission team.  In the voting phase, all five players simultaneously approve or reject the proposal.  A strict majority approves the team; five consecutive rejected proposals give the spies a win.  In the quest phase, mission-team members choose success/fail actions; spies may sabotage by choosing fail, while resistance players are forced to contribute success.  The first side to win three missions wins the game.

The observation is a compact public-history vector with shape $(128,1,1)$.  It contains played mission history, proposed team masks, vote records, mission outcomes, sabotage counts, the current round, phase, leader identity, whether the observing player is leader or on the proposed team, consecutive rejection count, current spy/resistance mission scores, current mission size, and private bits.  Private bits include whether the observing player is a spy and whether the observing player is currently on mission.  In the non-blind variant, spies additionally observe a partner-spy bitmask; in \texttt{avalon5\_blind}, this partner-spy bit is removed.  Thus all agents know their own role, but resistance players do not observe spy identities and blind spies do not observe the other spy.

The action space has size ten.  In the proposal phase, the leader's action indexes one of the valid mission-team combinations of the required size; non-leader actions are ignored.  In the voting phase, actions are thresholded into approve/reject votes.  In the quest phase, selected spies can play success or fail, while selected resistance actions are interpreted as success.  Rewards are team based: each mission gives $+1$ to the winning side and $-1$ to the losing side, and the terminal game bonus is $\pm 10$.  For D-BOS, role 0 is the primary shaper spy, role 1 is the second spy, and roles 2--4 are resistance players.  Avalon5 uses the MLP policy path with role-conditioned policy/value heads; adapter runs add a shaper-specific adapter and role-0 policy correction so the D-BOS gradient is restricted to the shaper role path.

\paragraph{Multi-agent CoinGame.}
The multi-agent CoinGame is a controlled mixed-motive hidden-role grid environment.  One fixed red shaper acts with a latent role $z \in \{0,1\}$, and there are three blue observers. Each episode is a toroidal grid game with grid size 5 and horizon 30.  At reset, all agents and a single coin are placed uniformly at random, and the red role is sampled uniformly.

Each agent observes a 10-dimensional vector:
\begin{equation}
[t/T,\; x_a,\; y_a,\; x_r,\; y_r,\; x_{b_1},\; y_{b_1},\; x_c,\; y_c,\; c],
\end{equation}
where $(x_a,y_a)$ is the observing agent position, $(x_r,y_r)$ is red's position, $(x_{b_1},y_{b_1})$ is the first blue observer position, $(x_c,y_c)$ is the coin position, and $c$ encodes the coin color together with the number of blue observers.  Coordinates are normalized before conversion to uint8.  The action space has four movement actions: up, down, left, and right, with wrap-around at grid boundaries.

Collecting any coin gives the collector $+1$.  If red collects a blue coin, the blue observers receive a total penalty of $-2$ shared equally across blue agents.  If red's latent role is altruistic, red also receives an additional $-1$ for collecting a blue coin.  If a blue observer collects a red coin, red receives $-2$.  The red return is the primary plotted score and blue score is the sum over blue observers.  D-BOS beliefs are over red's latent role; because blue observers share the same reward structure but maintain separate belief rows, CoinGame is a lightweight testbed for whether a matrix belief critic can aggregate multiple observer beliefs.

\subsection{Implementation Stabilization for the Belief Bound}
\label{app:stabilization}

The gradient bound in Theorem~\ref{thm:grad-error} assumes all belief coordinates stay above $b_{\min}>0$.  This is not only a technical assumption: hidden-role likelihoods can become extremely sharp, pushing posterior beliefs to the simplex boundary.  Near the boundary, $\log b^z$ has derivative $1/b^z$, so converting belief error to log-belief error introduces the $1/b_{\min}$ factors in Lemma~\ref{lem:log-conversion}.  If the implementation allowed $b^z \rightarrow 0$, the theoretical constant would become vacuous, and the numerical gradient through $\log b^z$ could become unstable.

In code, we handle this in three ways.  First, all Bayes updates are performed in log space, followed by a softmax, so the update remains differentiable and normalized:
\begin{equation}
b_{t+1} = \mathrm{softmax}(\ell_t + \log b_t).
\end{equation}
Second, in the MLP environments, we use an optional belief floor.  After each update, the posterior is mixed with a uniform distribution in \eqref{eq:belief_floor}:
\begin{equation}
\label{eq:belief_floor}
\tilde b_{t+1} = (1-\alpha_{\textrm{floor}})b_{t+1} + \alpha_{\textrm{floor}} \frac{\mathbf{1}}{|\mathcal{Z}|},
\end{equation}
where $\alpha_{\textrm{floor}}=\textrm{belief floor}$.  This implies $\tilde b_{t+1}^z \geq \alpha/|\mathcal{Z}|$, so in the theorem one may take $b_{\min}=\alpha/|\mathcal{Z}|$ for runs using this floor.  We also use likelihood temperature defined in \eqref{eq:likelihood_temp} with $\tau>1$ in some ablations, which prevents a single action likelihood from saturating the posterior too early.
\begin{equation}
\label{eq:likelihood_temp}
\ell_t \leftarrow \ell_t / \tau,
\end{equation}

Third, we also added a coefficient implementation that gates or clips unstable updates.  Windows whose terminal beliefs are already too certain are assigned zero coefficient, since these are exactly the windows where the softmax-Bayes Jacobian is least informative.  The coefficients can also be RMS-normalized and clipped.  This does not change the formal belief-error bound, but it prevents the applied policy correction from being dominated by rare high-magnitude windows. In RTG, the CNN path additionally clips the final shaping gradient to the same maximum norm scale used by PPO and optionally caps adaptive amplification relative to the PPO gradient norm.  Thus, the $b_{\min}$ dependence is handled both analytically, by assuming beliefs remain inside the simplex interior, and operationally, by keeping beliefs away from zero and clipping the applied shaping correction.

\subsection{Canonical and Estimated Observation Models}
\label{app:obs-models}

The Bayes likelihood in the paper is written as $\pi_\theta(a_i \mid h_{i,j,i}, z)$, where $h_{i,j,i}$ is the shaper's estimate of what observer $j$ believes the shaper observes.  In the canonical approximation, we set $h_{i,j,i}$ equal to the shaper's own local observation.  This means all observers share the same role-conditioned action likelihood at a given timestep, but maintain separate priors $b_t^{(j)}$.  Thus, the canonical setting shapes a matrix of observer beliefs, while the evidence model is shared unless observers have distinct priors or visibility weights.

In estimated-observation ablations, $h_{i,j,i}$ is predicted rather than copied.  In RTG, this is the existing backwards observation-prediction head in the deception model: during training, it is supervised with observer-local observations, and D-BOS uses the predicted observer observation in the differentiable Bayes likelihood.  In the MLP environments, the estimated-observation ablation uses a small LSTM predictor trained from rollout targets.  Its inputs are the shaper observation trajectory together with observer role and observer slot embeddings; its target is the observer's local observation. The D-BOS likelihood is given in Equation \eqref{eq:d-bos-likelihood}
\begin{equation}
\label{eq:d-bos-likelihood}
\ell_t^{j,z} = \log \pi_\theta(a_t^i \mid \hat o_{i,j,i,t}, z),
\end{equation}
Here, $\hat o_{i,j,i,t}$ is the predicted observer-conditioned observation.  We do not use the observer's true private observation directly for final claims; using it directly would be a privileged-information diagnostic rather than an inferred Theory-of-Mind model.

\subsection{Likelihood-Coefficient D-BOS Correction}
\label{app:coeff-correction}

The direct autograd form of D-BOS minimizes
\begin{equation}
\mathcal{L}_{\mathrm{BOS}}^k(\theta) =
-V_\phi(o_{t+k}, B_{t+k}),
\end{equation}
where $B_{t+k}$ is the full observer-by-role belief matrix obtained by chaining $k$ softmax-Bayes updates.  Backpropagating this loss through the belief chain is conceptually simple, but in practice, it can create noisy gradients through repeated softmax Jacobians.  Our coefficient implementation computes the endpoint sensitivity in likelihood space and then applies it as a stable surrogate policy loss.

For a window starting at $t$, let $\ell_s^{j,z}=\log \pi_\theta(a_s^{i}\mid \hat o_{i,j,i,s}, z)$ be the log-likelihood of the shaper's observed action under role hypothesis $z$ for observer $j$.  The belief chain is
\begin{equation}
B_{s+1} = U(B_s,\ell_s), \qquad s=t,\ldots,t+k-1,
\end{equation}
where $U$ applies the softmax-Bayes update row-wise over observer beliefs.  We first compute
\begin{equation}
c_s^{j,z}
= \frac{\partial [-V_\phi(o_{t+k}, B_{t+k})]}
        {\partial \ell_s^{j,z}}.
\end{equation}
These coefficients already contain the product of the endpoint value gradient, the intervening belief Jacobians, and the softmax-Bayes Jacobian.  We then detach $c_s$ and optimize the surrogate
\begin{equation}
\tilde{\mathcal{L}}_{\mathrm{BOS}}
= \frac{1}{k}\sum_{s=t}^{t+k-1}
  \frac{1}{|\mathcal{J}||\mathcal{Z}|}
  \sum_{j,z} \bar c_s^{j,z}\,\ell_s^{j,z},
\end{equation}
where $\bar c_s$ denotes the optionally gated, normalized, and clipped coefficient tensor.  Detaching the coefficients makes the correction first-order with respect to the policy parameters in the PPO update while preserving the direction computed by the differentiable belief planner at the rollout parameters.

For hidden-role shaping, adapter runs apply the D-BOS correction only to the shaper's role-conditioned policy path rather than to opponent-role policy rows. In coefficient mode, this is implemented by zeroing all coefficient entries except the shaper-role entry $z=0$ before normalization and clipping. This prevents the shaping term from directly modifying observer-role policies while keeping the base PPO training pipeline shared across methods.

The coefficient implementation logs the mean absolute coefficient, RMS coefficient, entropy-gate fraction, clipping fraction, approximate KL, pre/post clipping D-BOS gradient norms, and final injected D-BOS gradient norm.  These diagnostics are used to verify that the shaping term is present but not overwhelming the PPO gradient.

\subsection{Algorithm outlined}
\label{alg:outline}
\begin{algorithm}[t]
\caption{Differentiable Belief-based Opponent Shaping (D-BOS) with PPO}
\label{alg:bos-training}
\begin{algorithmic}[1]
\Require Role-conditioned policy $\pi_\theta(a\mid o,z)$; belief critic $V_\phi$; rollout length $T$; shaping horizon $k$; shaping weight $\lambda$; observers $\mathcal{J}$
\State Initialize per-observer beliefs $\{\hat{\mathbf{b}}_0^j\}_{j\in\mathcal{J}}$
\For{$t=0,\dots,T-1$}
    \State Sample shaper action $a_t \sim \pi_\theta(\cdot \mid o_t^{(i)}, z_i)$ and step the environment
    \For{each observer $j \in \mathcal{J}$}
        \State Choose observer proxy $\hat{o}_{i,j,i,t}$: estimated DeceptionModel proxy when available, otherwise canonical proxy
        \State Compute role likelihoods $\ell_t^{j,z} \gets \log \pi_\theta(a_t \mid \hat{o}_{i,j,i,t}, z)$ for all $z \in \mathcal{Z}$
        \State Update belief $\hat{\mathbf{b}}_{t+1}^j \gets \mathrm{softmax}(\ell_t^j + \log \hat{\mathbf{b}}_t^j)$
    \EndFor
    \State Store $(o_t,a_t,r_t,z_i,\log\pi_\theta(a_t\mid o_t^{(i)},z_i),\{\hat{\mathbf{b}}_t^j\}_{j\in\mathcal{J}})$ in the rollout buffer
\EndFor
\State Compute bootstrapped extrinsic returns $\{R_t\}_{t=0}^{T-1}$
\State Train $V_\phi$ by MSE regression on future endpoints:
\Statex \hspace{\algorithmicindent}$\displaystyle V_\phi(o_{t+k}^{(i)}, \hat{\mathbf{B}}_{t+k}) \approx R_{t+k}, \quad t=0,\dots,T-k-1$
\State Freeze $\phi$ and compute a rollout-level belief correction $g_{\mathrm{BOS}}$ from $\{-V_\phi(o_{t+k}^{(i)},\hat{\mathbf{B}}_{t+k})\}_{t=0}^{T-k-1}$
\State Clip, gate, or adaptively scale $g_{\mathrm{BOS}}$ using rollout diagnostics
\For{each PPO minibatch}
    \State Compute PPO gradient $g_{\mathrm{PPO}} \gets \nabla_\theta \mathcal{L}_{\mathrm{PPO}}$
    \State Apply the PPO update with the D-BOS correction to the shaper policy path:
    \Statex \hspace{\algorithmicindent}$\theta \leftarrow \theta - \alpha\!\left(g_{\mathrm{PPO}} + \mathrm{Apply}(g_{\mathrm{BOS}})\right)$
\EndFor
\end{algorithmic}
\end{algorithm}

\paragraph{Construction of $o_{t+k}$ and $B_{t+k}$.}
The implementation loops over valid rollout windows $t=0,\ldots,T-k-1$. For each environment and each observer $j$, it initializes the observer belief from the rollout buffer,
\begin{equation}
B_t = \{\mathbf{b}_t^j\}_{j\in\mathcal{J}},
\end{equation}
where \texttt{buf.per\_obs\_beliefs[t]} stores one role posterior per observer. It then applies exactly $k$ row-wise softmax-Bayes updates using the shaper action likelihoods at rollout times $t,\ldots,t+k-1$:
\begin{equation}
\mathbf{b}_{s+1}^j
= \mathrm{Reg}\!\left[
    \mathrm{softmax}\!\left(
        \ell_s^j + \log \mathbf{b}_{s}^{j}
    \right)
  \right],
\qquad s=t,\ldots,t+k-1,
\end{equation}
where $\ell_s^j=(\ell_s^{j,z})_{z\in\mathcal{Z}}$ and $\mathrm{Reg}$ denotes the belief floor/regularization used to avoid simplex-boundary saturation. After the final update, the code flattens the matrix $B_{t+k}\in\mathbb{R}^{|\mathcal{J}|\times|\mathcal{Z}|}$ into a vector and evaluates the belief critic on the shaper's future observation:
\begin{equation}
o_{t+k} = \texttt{buf.obs[t+k, shaper\_index]},
\qquad
V_\phi(o_{t+k}, \mathrm{vec}(B_{t+k})).
\end{equation}
The belief critic is trained on the same indexing convention: for role-0/shaper samples, it regresses $V_\phi(o_{t+k},B_{t+k})$ to the rollout return stored at time $t+k$. Thus the shaping query and critic target both use the same future pair $(o_{t+k},B_{t+k})$.

\subsection{Proof of Proposition~\ref{prop:sufficient-statistic}: Belief is a sufficient statistic}
\label{proof:sufficient-statistic}

\begin{proof}
Since $\pi_{\text{opp}}^*(a \mid o_t, z)$ is independent of $\theta_1$ by assumption, the only dependence on $\theta_1$ in Eq.~\ref{eq:opp-policy} is through $b_t^z$. Therefore holding $b_t$ fixed, the marginal distribution $\pi_{\text{opp}}^*(a \mid o_t)$ is fully determined and does not depend on $\theta_1$, meaning that $b_t$ is a sufficient statistic.\\
Differentiating Eq.~\ref{eq:opp-policy} with respect to $\theta_1$ using the product rule:$$\frac{\partial\pi^*_{opp}(a|o_t)}{\partial\theta_1}=\sum_z\frac{\partial b_t^z}{\partial\theta_1}\cdot \pi^*_{opp}(a|o_t,z)+\sum_z b_t^z \cdot \underbrace{{\frac{\partial\pi^*_{opp}(a|o_t,z)}{\partial\theta_1}}}_{=0 \text{ (by assumption)}}
= \sum_z \frac{\partial b_t^z}{\partial \theta_1} \cdot \pi_{\text{opp}}^*(a \mid o_t, z)$$
The gradient flows entirely through  $b_t$.
\end{proof}

\subsection{Proof of Theorem~\ref{thm:meta-shaping}: Opponent shaping via differentiable opponent state}
\label{proof:meta-shaping}

\begin{proof}
Since $J_1$ is differentiable in both arguments and $\eta_t$ depends on $\theta_1$ through the differentiable recursion in Eq. \ref{eq:opp-state-update}, the chain rule gives
\begin{equation}
    \nabla_{\theta_1} J_1(\theta_1, \eta_t)
    = \frac{\partial J_1}{\partial \theta_1}\bigg|_{\eta_t} + \left(\nabla_{\theta_1}\eta_t\right)^\top \nabla_{\eta_t} J_1(\theta_1, \eta_t),
\end{equation}
which is Eq.~\ref{eq:meta-gradient-decomp}.

To compute $\nabla_{\theta_1}\eta_t$, differentiate the state recursion (Eq.~\ref{eq:opp-state-update}):
\begin{equation}
    \nabla_{\theta_1}\eta_{t+1}
    = \frac{\partial U}{\partial \eta_t}\,\nabla_{\theta_1}\eta_t + \frac{\partial U}{\partial \theta_1},
\end{equation}
where the partial derivatives are evaluated at $(\eta_t, \theta_1, \xi_t)$. We Initialize with $\nabla_{\theta_1}\eta_0=0$ since the initial state $\eta_0$ is independent of $\theta_1$. Unrolling this recursion yields an explicit expression for $\nabla_{\theta_1}\eta_t$ in terms of Jacobians of $U$:
\begin{equation}
    \nabla_{\theta_1}\eta_{t+1}=\sum_{i=0}^t
    \left(\prod_{j=i+1}^t\frac{\partial U}{\partial\eta_j}\right)\frac{\partial U}{\partial \theta_1} \bigg|_{t=i}  
\end{equation}
\end{proof}

\subsection{Proof of Lemma~\ref{lem:softmax-lipschitz}: Softmax Lipschitz property}
\label{proof:softmax-lipschitz}

\begin{proof}

Let $f(x) = \mathrm{softmax}(x)$. For any $x$ write $b = f(x)$. The Jacobian 
of $f$ at $x$ has entries
\begin{equation}
J_{z', z}(x) = b^{z'}\,(\mathbf{1}[z' = z] - b^z),
\end{equation}
or equivalently $J(x) = \mathrm{diag}(b) - bb^\top$.

\textbf{Bounding $\|J(x)\|_{\infty \to 1}$.} For any vector $v \in \mathbb{R}^{|\mathcal{Z}|}$,
\begin{equation}
J(x)\,v = \mathrm{diag}(b)\,v - b\,(b^\top v) = b \odot v - (b^\top v)\,b,
\end{equation}
where $\odot$ denotes component-wise multiplication. Taking the 1-norm and 
applying the triangle inequality,
\begin{equation}
\|J(x) v\|_1 \leq \|b \odot v\|_1 + |b^\top v| \cdot \|b\|_1.
\end{equation}
Each term: $\|b \odot v\|_1 = \sum_z b^z |v^z| \leq \|b\|_1 \cdot \|v\|_\infty 
= \|v\|_\infty$ (since $\|b\|_1 = 1$ for a probability vector) and 
$|b^\top v| \leq \|b\|_1 \|v\|_\infty = \|v\|_\infty$ by H\"older. Combining,
\begin{equation}
\|J(x) v\|_1 \leq \|v\|_\infty + \|v\|_\infty = 2\|v\|_\infty.
\end{equation}
Hence $\|J(x)\|_{\infty \to 1} \leq 2$ at every $x$.

Define $z(t) = y + t(x - y)$. By the fundamental theorem 
of calculus,
\begin{align}
\|f(x) - f(y)\|_1 
&= \left\|\int_0^1 J(z(t))(x - y)\,dt\right\|_1 \notag \\
&\leq \int_0^1 \|J(z(t))(x - y)\|_1\,dt 
\quad \text{(triangle inequality)} \\
&\leq \int_0^1 2\|x - y\|_\infty\,dt 
\quad \text{(operator norm bound)} \\
&= 2\|x - y\|_\infty.
\end{align}
\end{proof}

\subsection{Proof of Proposition~\ref{prop:belief-error}: \texorpdfstring{$k$}{k}-step belief error}
\label{proof:k-step-belief-error}

\begin{proof}

Recall the per-step log-likelihood vector
\begin{equation}
\ell_r := \log\pi_\theta(a_r \mid o_r, \cdot) \in \mathbb{R}^{|\mathcal{Z}|},
\end{equation}
whose $z$-th entry $\ell_r^z = \log\pi_\theta(a_r \mid o_r, z)$ is the 
log-probability that role $z$ would have taken action $a_r$ given observation 
$o_r$. The approximate version $\hat\ell_r := \log\pi_\theta(a_r \mid \hat 
o_r, \cdot)$ uses RED's second-order observation estimate $\hat o_r$ in 
place of the observer's true estimate $o_r$. 

Define the cumulative log-likelihood vectors
\begin{equation}
L_{t+k} := \log b_t + \sum_{i=0}^{k-1} \ell_{t+i}, 
\qquad 
\hat L_{t+k} := \log b_t + \sum_{i=0}^{k-1} \hat \ell_{t+i},
\end{equation}
where the starting log-belief $\log b_t$ is shared (since $b_t = \hat b_t$ by 
assumption). Applying $k$ softmax-Bayes updates is 
equivalent to a single softmax of the cumulative logits:
\begin{equation}
b_{t+k} = \mathrm{softmax}(L_{t+k}), 
\qquad 
\hat b_{t+k} = \mathrm{softmax}(\hat L_{t+k}).
\end{equation}

Applying Lemma \ref{lem:softmax-lipschitz},
\begin{equation}
\|\hat b_{t+k} - b_{t+k}\|_1 
= \|\mathrm{softmax}(\hat L_{t+k}) - \mathrm{softmax}(L_{t+k})\|_1 
\leq 2 \|\hat L_{t+k} - L_{t+k}\|_\infty.
\end{equation}
The cumulative difference is
\begin{equation}
\hat L_{t+k} - L_{t+k} = \sum_{i=0}^{k-1}\big(\hat \ell_{t+i} - \ell_{t+i}\big),
\end{equation}
so by the triangle inequality and the per-step error bound $\|\hat \ell_{t+i} - 
\ell_{t+i}\|_\infty \leq \varepsilon_{t+i}$ by definition \ref{def:log-likelihood-error},
\begin{equation}
\|\hat L_{t+k} - L_{t+k}\|_\infty 
\leq \sum_{i=0}^{k-1} \|\hat \ell_{t+i} - \ell_{t+i}\|_\infty 
\leq \sum_{i=0}^{k-1} \varepsilon_{t+i}.
\end{equation}
Combining,
\begin{equation}
\|\hat b_{t+k} - b_{t+k}\|_1 \leq 2 \sum_{i=0}^{k-1} \varepsilon_{t+i} 
\leq 2 k \varepsilon
\end{equation}
under the uniform bound $\varepsilon = \max_t \varepsilon_t$.
\end{proof}

% \end{proof}

\subsection{Proof of Theorem~\ref{thm:grad-error}: Gradient error from ToM approximation}
\label{app:thm2-proof}

We prove Theorem 2 in stages. First we establish two auxiliary results: a log-conversion lemma and a trajectory-drift corollary. We then bound the magnitude of each factor in the chain-rule 
expansion of $\nabla_\theta L^{(j)}$.

Throughout, we drop the observer index $(j)$ for compactness, fixing one observer 
$j$. The bound applies per-observer, aggregation across observers is handled 
separately. We use the norm conventions: $\|\cdot\|_1$ on belief vectors, 
$\|\cdot\|_\infty$ on log-beliefs and log-likelihood vectors, $\|\cdot\|_2$ on 
parameter-space vectors.

\begin{lemma}[Log conversion lemma]
\label{lem:log-conversion}
Let $b, \hat b \in \Delta^{|\mathcal{Z}|}$ satisfy $b^z \geq b_{\min}$ and 
$\hat b^z \geq b_{\min}$ for all $z \in \mathcal{Z}$, where $b_{\min} > 0$. Then
\begin{equation}
\|\log \hat b - \log b\|_\infty \leq \frac{1}{b_{\min}} \|\hat b - b\|_1.
\end{equation}
\end{lemma}

\begin{proof}
Let $z^* \in \arg\max_z |\log \hat b^z - \log b^z|$, so that
\begin{equation}
\|\log \hat b - \log b\|_\infty = |\log \hat b^{z^*} - \log b^{z^*}|.
\end{equation}
Both $b^{z^*}$ and $\hat b^{z^*}$ lie in $[b_{\min}, 1]$. By the Mean Value 
Theorem applied to $f(x) = \log x$ on this interval, there exists 
$c \in [b_{\min}, 1]$ such that
\begin{equation}
\log \hat b^{z^*} - \log b^{z^*} = \frac{1}{c} \big(\hat b^{z^*} - b^{z^*}\big).
\end{equation}
Taking absolute values and using $c \geq b_{\min}$:
\begin{equation}
|\log \hat b^{z^*} - \log b^{z^*}| 
= \frac{1}{c} |\hat b^{z^*} - b^{z^*}|
\leq \frac{1}{b_{\min}} |\hat b^{z^*} - b^{z^*}|.
\end{equation}
Finally, the absolute difference at any single coordinate is bounded by the 
total $\ell_1$ difference:
\begin{equation}
|\hat b^{z^*} - b^{z^*}| \leq \sum_{z} |\hat b^{z} - b^{z}| = \|\hat b - b\|_1.
\end{equation}
Combining these,
\begin{equation}
\|\log \hat b - \log b\|_\infty \leq \frac{1}{b_{\min}} \|\hat b - b\|_1.
\qedhere
\end{equation}

% \paragraph{Remark.} \textcolor{red}{TODO: addidng an explanation for $b_\text{min}$ assumption could help.}

Proposition~\ref{prop:belief-error} establishes that for the endpoint of a $k$-steps,
\begin{equation}
\|\hat b_{t+k} - b_{t+k}\|_1 \leq 2 \sum_{s=0}^{k-1} \varepsilon_{t+s} 
\leq 2 k \varepsilon,
\end{equation}
where $\varepsilon = \max_t \varepsilon_t$. The same induction applies at every 
intermediate step, giving the following trajectory-wide drift bound.

\begin{corollary}[Log-belief trajectory drift]
\label{cor:traj-drift}
Under the assumptions of Proposition~\ref{prop:belief-error}, for every $m \in \{0, 1, \dots, k\}$,
\begin{equation}
\|\hat b_{t+m} - b_{t+m}\|_1 \leq 2 \sum_{s=0}^{m-1} \varepsilon_{t+s} 
\leq 2 m \varepsilon.
\end{equation}
\end{corollary}

\begin{proof}
Apply Proposition~\ref{prop:belief-error} with horizon $m$ in place of $k$. The induction in the proof of Proposition~\ref{prop:belief-error} makes no use of $k$ being distinguished from any other horizon length
\end{proof}

Combining Corollary~\ref{cor:traj-drift} with Lemma~\ref{lem:log-conversion} 
gives the corresponding bound in $\infty$-norm of log-beliefs, which we use in subsequent steps.

\begin{corollary}
\label{cor:log-traj-drift}
Under the assumptions of Proposition~\ref{prop:belief-error} and Theorem~\ref{thm:grad-error}, for every  $m \in {0, 1, ..., k}$
\begin{equation}
\|\log \hat b_{t+m} - \log b_{t+m}\|_\infty \leq \frac{2 m \varepsilon}{b_{\min}}.
\end{equation}
\end{corollary}

\begin{proof}
Apply Lemma~\ref{lem:log-conversion} to $b_{t+m}$ and $\hat b_{t+m}$:
\begin{equation}
\|\log \hat b_{t+m} - \log b_{t+m}\|_\infty 
\leq \frac{1}{b_{\min}} \|\hat b_{t+m} - b_{t+m}\|_1
\leq \frac{2 m \varepsilon}{b_{\min}}.
\qedhere
\end{equation}
\end{proof}

\subsubsection{Chain-rule expansion of \texorpdfstring{$\nabla_\theta L$}{∇L}}

We expand $\nabla_\theta L = -\nabla_\theta V_\phi(o_{t+k}, b_{t+k})$ using the 
chain rule. The parameter $\theta$ enters the loss through $k$ separate 
channels, one per step in the window: at step $s \in \{t, \dots, t+k-1\}$, 
$\theta$ controls the log-likelihood vector
\begin{equation}
\ell_s := \log \pi_\theta(a_s \mid o_s, \cdot) \in \mathbb{R}^{|\mathcal{Z}|},
\end{equation}
which then propagates through softmax-Bayes updates 
$b_{r+1} = \mathrm{softmax}(\ell_r + \log b_r)$ until reaching $b_{t+k}$, 
which feeds into $V_\phi$.

By the multivariate chain rule, summing over all $k$ entry points,
\begin{equation}
\label{eq:chain-rule}
\nabla_\theta L 
= \sum_{s=t}^{t+k-1} 
\frac{\partial L}{\partial b_{t+k}} 
\cdot \left( \prod_{r=s+1}^{t+k-1} \frac{\partial b_{r+1}}{\partial b_r} \right) 
\cdot \frac{\partial b_{s+1}}{\partial \ell_s} 
\cdot \nabla_\theta \ell_s.
\end{equation}
For compactness, denote the four kinds of factors as
\begin{align}
g &:= \frac{\partial L}{\partial b_{t+k}} 
&& \text{(endpoint gradient),} \\
J_r &:= \frac{\partial b_{r+1}}{\partial b_r}
&& \text{(belief-to-belief Jacobian),} \\
S_s &:= \frac{\partial b_{s+1}}{\partial \ell_s}
&& \text{(softmax Jacobian),} \\
G_s &:= \nabla_\theta \ell_s
&& \text{(policy-gradient matrix).}
\end{align}
and let
\begin{equation}
\Pi_s := \prod_{r=s+1}^{t+k} J_r
\end{equation}
denote the belief-Jacobian chain from step $s+1$ to step $t+k$. Equation~\ref{eq:chain-rule} 
can compactly be written as
\begin{equation}
\label{eq:chain-rule-compact}
\nabla_\theta L = \sum_{s=t}^{t+k-1} g \cdot \Pi_s \cdot S_s \cdot G_s.
\end{equation}

The approximate gradient $\nabla_\theta L_{\text{approx}}$ has the same form 
as Eq.~\ref{eq:chain-rule-compact}, but every factor is evaluated along the 
approximate trajectory $(\hat b_s)_{s=t}^{t+k}$ generated using $\hat o_s$ in 
place of $o_s$. We denote the approximate factors with hats: 
$\hat g, \hat J_r, \hat S_s, \hat G_s, \hat \Pi_s$. Both trajectories share the 
same starting belief, $\hat b_t = b_t$.

\subsubsection{Magnitude bounds on the chain-rule factors}

We bound each of the chain-rule factors in the appropriate operator norm.

\begin{proposition}[Magnitude bounds]
\label{prop:magnitude}
Under the assumptions of Theorem~\ref{thm:grad-error}, the factors of 
Eq.~\ref{eq:chain-rule-compact} satisfy
\begin{align}
\|g\|_\infty &\leq L_V, \\
\|S_s\|_{\infty \to 1} &\leq 2, \\
\|G_s\|_{2 \to \infty} &\leq G_\pi.
\end{align}
The same bounds hold for $\hat g, \hat S_s, \hat G_s$ along the approximate 
trajectory.
\end{proposition}

\begin{proof}

\textbf{Bound on $g$.} Assumption~(4) states that $V_\phi(o, \cdot)$ is 
$L_V$-Lipschitz in $\|\cdot\|_1$:
\begin{equation}
|V_\phi(o, b) - V_\phi(o, b')| \leq L_V \|b - b'\|_1 
\quad \forall b, b' \in \Delta^{|\mathcal{Z}|}.
\end{equation}
By the duality between $\|\cdot\|_1$ and $\|\cdot\|_\infty$, the gradient 
$g = \partial L / \partial b_{t+k} = -\partial V_\phi / \partial b_{t+k}$ 
satisfies
\begin{equation}
\|g\|_\infty 
= \sup_{\|\delta b\|_1 \leq 1} |g \cdot \delta b|
\leq L_V.
\end{equation}

\textbf{Bound on $S_s$.} Recall $b_{s+1} = \mathrm{softmax}(\ell_s + \log b_s)$, 
so $S_s = \partial b_{s+1} / \partial \ell_s$ is the Jacobian of $\mathrm{softmax}$ 
evaluated at $\ell_s + \log b_s$. Lemma~\ref{lem:softmax-lipschitz} states
\begin{equation}
\|\mathrm{softmax}(x) - \mathrm{softmax}(y)\|_1 \leq 2 \|x - y\|_\infty 
\quad \forall x, y \in \mathbb{R}^{|\mathcal{Z}|},
\end{equation}
i.e., $\mathrm{softmax}$ is $2$-Lipschitz from $\|\cdot\|_\infty$ to $\|\cdot\|_1$. 
For any direction $v \in \mathbb{R}^{|\mathcal{Z}|}$ and $\eta > 0$, applying this 
to $x = u + \eta v$ and $y = u$ gives
\begin{equation}
\|\mathrm{softmax}(u + \eta v) - \mathrm{softmax}(u)\|_1 \leq 2 \eta \|v\|_\infty.
\end{equation}
Dividing by $\eta$ and taking $\eta \to 0$ yields 
$\|J_{\mathrm{softmax}}(u) \, v\|_1 \leq 2 \|v\|_\infty$ for all $v$, so the 
softmax Jacobian satisfies $\|J_{\mathrm{softmax}}(u)\|_{\infty \to 1} \leq 2$ 
at every input $u$. In particular, evaluating at $u = \ell_s + \log b_s$,
\begin{equation}
\|S_s\|_{\infty \to 1} \leq 2.
\end{equation}

\textbf{Bound on $G_s$.} The matrix $G_s \in \mathbb{R}^{|\mathcal{Z}| \times d_\theta}$ 
has rows $G_s^{z, :} = \nabla_\theta \log \pi_\theta(a_s \mid o_s, z) \in 
\mathbb{R}^{d_\theta}$ indexed by role $z$. For any $x \in \mathbb{R}^{d_\theta}$, 
the $z$-th coordinate of $G_s x$ is the dot product $G_s^{z,:} \cdot x$, so
\begin{equation}
\|G_s x\|_\infty = \max_z |G_s^{z, :} \cdot x|.
\end{equation}
Taking the supremum over $\|x\|_2 \leq 1$ and exchanging $\sup$ and $\max$,
\begin{equation}
\|G_s\|_{2 \to \infty} 
= \sup_{\|x\|_2 \leq 1} \max_z |G_s^{z, :} \cdot x|
= \max_z \, \sup_{\|x\|_2 \leq 1} |G_s^{z, :} \cdot x|.
\end{equation}
The inner supremum is the dual norm of $G_s^{z,:}$ under $\|\cdot\|_2$; since 
$\|\cdot\|_2$ is self-dual, Cauchy--Schwarz gives 
$\sup_{\|x\|_2 \leq 1} |G_s^{z,:} \cdot x| = \|G_s^{z,:}\|_2$, with equality 
attained at $x = G_s^{z,:}/\|G_s^{z,:}\|_2$. Hence
\begin{equation}
\|G_s\|_{2 \to \infty} = \max_z \|G_s^{z,:}\|_2 
= \max_z \|\nabla_\theta \log\pi_\theta(a_s \mid o_s, z)\|_2 \leq G_\pi,
\end{equation}
where the last inequality follows from the definition of $G_\pi$.

The bounds for the approximate factors $\hat g, \hat S_s, \hat G_s$ follow from 
the same arguments applied along the approximate trajectory.
\end{proof}

\subsubsection{Decomposition of the gradient error}
\label{sec:decomposition}

Define the per-summand contributions to $\nabla_\theta L$ and 
$\nabla_\theta L_{\text{approx}}$ as
\begin{equation}
T_s := g \cdot \Pi_s \cdot S_s \cdot G_s, \qquad
\hat T_s := \hat g \cdot \hat \Pi_s \cdot \hat S_s \cdot \hat G_s,
\end{equation}
so that $\nabla_\theta L_{\text{approx}} - \nabla_\theta L_{\text{true}} = 
\sum_{s=t}^{t+k-1} (\hat T_s - T_s)$.

Within each summand, the hat and no-hat versions differ in all four factors. 
To isolate one source of error at a time, we telescope by swapping one factor 
at a time, from left to right. Adding and subtracting intermediate hybrids,
\begin{align}
\hat T_s - T_s 
&= \underbrace{(\hat g - g) \, \hat \Pi_s \hat S_s \hat G_s}_{\text{(A)}} 
+ \underbrace{g \, (\hat \Pi_s - \Pi_s) \hat S_s \hat G_s}_{\text{(B)}} \notag \\
&\quad + \underbrace{g \, \Pi_s (\hat S_s - S_s) \hat G_s}_{\text{(C)}}
+ \underbrace{g \, \Pi_s S_s (\hat G_s - G_s)}_{\text{(D)}}.
\label{eq:four-piece}
\end{align}
Each piece isolates the perturbation of a single factor. 
Summing over $s$ and applying the triangle inequality in $\|\cdot\|_2$,
\begin{equation}
\label{eq:total-error}
\big\|\nabla_\theta L_{\text{approx}} - \nabla_\theta L_{\text{true}}\big\|_2 
\leq \sum_{s=t}^{t+k-1} \Big( \|\text{(A)}_s\|_2 + \|\text{(B)}_s\|_2 
+ \|\text{(C)}_s\|_2 + \|\text{(D)}_s\|_2 \Big).
\end{equation}
The four pieces are bounded individually in the following subsections.

\subsubsection{Bound on piece (D): policy-gradient error}
\label{sec:piece-D}

Piece (D) isolates the error from evaluating the policy gradient at $\hat o_s$ 
instead of $o_s$:
\begin{equation}
(D)_s = g \cdot \Pi_s \cdot S_s \cdot (\hat G_s - G_s).
\end{equation}
By submultiplicativity of operator norms,
\begin{equation}
\label{eq:D-submult}
\|(D)_s\|_2 \leq \|g\|_\infty \cdot \|\Pi_s\|_{1 \to 1} 
\cdot \|S_s\|_{\infty \to 1} \cdot \|\hat G_s - G_s\|_{2 \to \infty}.
\end{equation}

The first three factors are bounded via Proposition~3 and Section~\ref{sec:piece-B}. We bound the difference factor \ref{sec:piece-B}
$\|\hat G_s - G_s\|_{2 \to \infty}$ here.

The matrix $\hat G_s - G_s \in \mathbb{R}^{|\mathcal{Z}| \times d_\theta}$ has 
rows
\begin{equation}
(\hat G_s - G_s)^{z, :} = \nabla_\theta \log\pi_\theta(a_s | \hat o_s, z) 
- \nabla_\theta \log\pi_\theta(a_s | o_s, z) \in \mathbb{R}^{d_\theta}.
\end{equation}
By Assumption~(2), each row satisfies
\begin{equation}
\|(\hat G_s - G_s)^{z, :}\|_2 \leq L_\pi \|\hat o_s - o_s\|_\infty.
\end{equation}
Following the same argument as in the proof of Proposition~3 (the $2 \to \infty$ 
operator norm equals the maximum row 2-norm),
\begin{equation}
\|\hat G_s - G_s\|_{2 \to \infty} = \max_z \|(\hat G_s - G_s)^{z, :}\|_2 
\leq L_\pi \|\hat o_s - o_s\|_\infty.
\end{equation}

Combining with the magnitude bounds from Proposition~3 ($\|g\|_\infty \leq L_V$, 
$\|S_s\|_{\infty \to 1} \leq 2$), Eq.~\ref{eq:D-submult} becomes
\begin{equation}
\label{eq:piece-D-final}
\|(D)_s\|_2 \leq 2 L_V \cdot \|\Pi_s\|_{1 \to 1} \cdot L_\pi \|\hat o_s - o_s\|_\infty,
\end{equation}
where the bound on $\|\Pi_s\|_{1 \to 1}$ is established in 
Section~\ref{sec:piece-B}.

\subsubsection{Bound on piece (A): endpoint error}
\label{sec:piece-A}

By submultiplicativity of operator norms,

\begin{equation}
\label{eq:A-submult}
\|(A)_s\|_2 \leq \|\hat g - g\|_\infty \cdot \|\hat \Pi_s\|_{1 \to 1} 
\cdot \|\hat S_s\|_{\infty \to 1} \cdot \|\hat G_s\|_{2 \to \infty}.
\end{equation}
Bounding $\|\hat g - g\|_\infty$ uses Assumption 5 that $\nabla_b V_\phi$ itself be 
Lipschitz in $b$, with some constant $L_g$. Under this assumption,
\begin{equation}
\|\hat g - g\|_\infty 
\leq L_g \|\hat b_{t+k} - b_{t+k}\|_1 
\leq 2 k \varepsilon \cdot L_g
\end{equation}
by Corollary~3.

Substituting into Eq.~\ref{eq:A-submult} and using Proposition~3,
\begin{equation}
\label{eq:piece-A-final}
\|(A)_s\|_2 \leq 2 k \varepsilon L_g \cdot \|\hat \Pi_s\|_{1 \to 1} 
\cdot 2 \cdot G_\pi 
= 4 k \varepsilon L_g G_\pi \cdot \|\hat \Pi_s\|_{1 \to 1},
\end{equation}
where the bound on $\|\hat \Pi_s\|_{1 \to 1}$ is established in 
Section~\ref{sec:piece-B}.

\subsubsection{Bound on piece (C): softmax error}
\label{sec:piece-C}

By submultiplicativity,
\begin{equation}
\|(C)_s\|_2 \leq \|g\|_\infty \cdot \|\Pi_s\|_{1 \to 1} 
\cdot \|\hat S_s - S_s\|_{\infty \to 1} \cdot \|\hat G_s\|_{2 \to \infty}.
\end{equation}

The softmax Jacobian admits the closed form $S_s = \mathrm{diag}(b_{s+1}) - 
b_{s+1} b_{s+1}^\top$. Letting $\delta := \hat b_{s+1} - b_{s+1}$ and using 
$\hat x \hat x^\top - x x^\top = \hat x \delta^\top + \delta x^\top$,
\begin{equation}
\hat S_s - S_s = \mathrm{diag}(\delta) - \hat b_{s+1}\delta^\top 
- \delta b_{s+1}^\top.
\end{equation}

We bound each of the three terms in $\infty \to 1$ operator norm. For the 
diagonal term, $\|\mathrm{diag}(\delta) x\|_1 = \sum_z |\delta_z x_z| 
\leq \|\delta\|_1 \|x\|_\infty$. For each rank-one term $u v^\top$, the output 
$u(v^\top x)$ has 1-norm $\|u\|_1 \cdot |v^\top x| \leq \|u\|_1 \|v\|_1 \|x\|_\infty$ 
by H\"older's inequality, and since $\|b_{s+1}\|_1 = \|\hat b_{s+1}\|_1 = 1$ 
(probability vectors), each rank-one term contributes at most $\|\delta\|_1$. 
By the triangle inequality,
\begin{equation}
\|\hat S_s - S_s\|_{\infty \to 1} \leq 3\|\hat b_{s+1} - b_{s+1}\|_1 
\leq 6 k \varepsilon
\end{equation}
where the second inequality is Corollary \ref{cor:traj-drift}.

Combining with Proposition~3,
\begin{equation}
\label{eq:piece-C-final}
\|(C)_s\|_2 \leq 6 k \varepsilon L_V G_\pi \cdot \|\Pi_s\|_{1 \to 1}.
\end{equation}

\subsubsection{Bound on piece (B): Jacobian-chain error}
\label{sec:piece-B}

The chain $\Pi_s$ is a product of $t+k-s-1$ belief-to-belief Jacobians. We establish that the chain $\Pi_s$ \emph{collapses algebraically} to a single Jacobian whose form depends only on the starting belief $b_{s+1}$ and the ending belief $b_{t+k}$, with all intermediate beliefs disappearing. This collapse keeps 
the dependence on $1/b_{\min}$ polynomial.

\begin{lemma}[Jacobian-chain collapse]
\label{lem:chain-collapse}
For any $r$ and $m \geq 0$,
\begin{equation}
J_{r+m} \cdot J_{r+m-1} \cdots J_{r+1} \cdot J_r 
= D_{r+m+1}\,(I - \mathbf{1}\,b_{r+m+1}^\top)\,D_r^{-1}.
\end{equation}
In particular, applying this with $r = s+1$ and $m = t+k-s-2$,
\begin{equation}
\label{eq:Pi-closed-form}
\Pi_s = D_{t+k}\,(I - \mathbf{1}\,b_{t+k}^\top)\,D_{s+1}^{-1}, 
\qquad 
\hat \Pi_s = \hat D_{t+k}\,(I - \mathbf{1}\,\hat b_{t+k}^\top)\,\hat D_{s+1}^{-1},
\end{equation}
where $D_r := \mathrm{diag}(b_r)$ and $\hat D_r := \mathrm{diag}(\hat b_r)$.
\end{lemma}

\begin{proof}
By induction on $m$.

\emph{Base case ($m = 0$).} The claim reads
\begin{equation}
J_r = D_{r+1}(I - \mathbf{1}\,b_{r+1}^\top)\,D_r^{-1}.
\end{equation}
We derive this from the chain rule. The forward computation is 
$b_{r+1} = \mathrm{softmax}(\ell_r + \log b_r)$, where 
$\ell_r := \log\pi_\theta(a_r \mid o_r, \cdot) \in \mathbb{R}^{|\mathcal{Z}|}$ 
is the log-likelihood vector at step $r$. Note that $\ell_r$ depends on 
$\theta$, $a_r$, and $o_r$ but not on $b_r$. Writing $u := \ell_r + \log b_r$ 
for the input to softmax, the dependency chain is
\begin{equation}
b_r \;\longrightarrow\; \log b_r \;\longrightarrow\; u = \ell_r + \log b_r 
\;\longrightarrow\; b_{r+1} = \mathrm{softmax}(u),
\end{equation}
giving by the chain rule
\begin{equation}
J_r = \frac{\partial b_{r+1}}{\partial b_r} 
= \frac{\partial b_{r+1}}{\partial u} \cdot \frac{\partial u}{\partial(\log b_r)} 
\cdot \frac{\partial(\log b_r)}{\partial b_r}.
\end{equation}
We compute each factor.

The map $b_r \mapsto \log b_r$ acts component-wise with 
$d \log x / dx = 1/x$, so its Jacobian is the diagonal matrix
\begin{equation}
\frac{\partial(\log b_r)}{\partial b_r} = \mathrm{diag}(1/b_r) = D_r^{-1}.
\end{equation}

Since $\ell_r$ does not depend on $b_r$,
\begin{equation}
\frac{\partial u}{\partial(\log b_r)} 
= \frac{\partial \ell_r}{\partial(\log b_r)} 
+ \frac{\partial(\log b_r)}{\partial(\log b_r)} 
= 0 + I = I.
\end{equation}

The standard softmax Jacobian formula gives, for 
$b_{r+1} = \mathrm{softmax}(u)$,
\begin{equation}
\frac{\partial b_{r+1}}{\partial u} 
= \mathrm{diag}(b_{r+1}) - b_{r+1} b_{r+1}^\top.
\end{equation}
Using $b_{r+1} = \mathrm{diag}(b_{r+1})\mathbf{1}$, we factor 
$b_{r+1} b_{r+1}^\top = \mathrm{diag}(b_{r+1})\,\mathbf{1}\,b_{r+1}^\top$, hence
\begin{equation}
\frac{\partial b_{r+1}}{\partial u} 
= \mathrm{diag}(b_{r+1})\,\big[I - \mathbf{1}\,b_{r+1}^\top\big] 
= D_{r+1}(I - \mathbf{1}\,b_{r+1}^\top).
\end{equation}

\emph{Combining the three factors,}
\begin{equation}
J_r = D_{r+1}(I - \mathbf{1}\,b_{r+1}^\top) \cdot I \cdot D_r^{-1} 
= D_{r+1}\,(I - \mathbf{1}\,b_{r+1}^\top)\,D_r^{-1}.
\end{equation}
This establishes the base case.

\emph{Inductive step.} Assume the claim holds for $m$. Multiplying on the 
left by $J_{r+m+1} = D_{r+m+2}(I - \mathbf{1}b_{r+m+2}^\top)D_{r+m+1}^{-1}$:
\begin{align}
J_{r+m+1} \cdot \big(J_{r+m} \cdots J_r\big)
&= D_{r+m+2}(I - \mathbf{1}b_{r+m+2}^\top)\,D_{r+m+1}^{-1} \cdot 
D_{r+m+1}(I - \mathbf{1}b_{r+m+1}^\top)\,D_r^{-1} \notag \\
&= D_{r+m+2}(I - \mathbf{1}b_{r+m+2}^\top)(I - \mathbf{1}b_{r+m+1}^\top)\,D_r^{-1},
\end{align}
where $D_{r+m+1}^{-1} D_{r+m+1} = I$ cancels. We claim that
\begin{equation}
\label{eq:rank-one-idempotent}
(I - \mathbf{1}b_{r+m+2}^\top)(I - \mathbf{1}b_{r+m+1}^\top) 
= I - \mathbf{1}b_{r+m+2}^\top.
\end{equation}
For any $x$, the right-hand factor gives $(I - \mathbf{1}b_{r+m+1}^\top)x = 
x - \mathbf{1}(b_{r+m+1}^\top x)$. Applying the left-hand factor and using 
$(I - \mathbf{1}b_{r+m+2}^\top)\mathbf{1} = \mathbf{1} - \mathbf{1}\,(b_{r+m+2}^\top \mathbf{1}) 
= 0$ (since $b_{r+m+2}^\top \mathbf{1} = 1$), the second term vanishes, leaving 
$(I - \mathbf{1}b_{r+m+2}^\top)x$. Equation~\ref{eq:rank-one-idempotent} 
follows. Substituting,
\begin{equation}
J_{r+m+1} \cdots J_r = D_{r+m+2}(I - \mathbf{1}b_{r+m+2}^\top)\,D_r^{-1},
\end{equation}
which is the inductive claim for $m+1$.
\end{proof}

The closed form in Eq.\ref{eq:Pi-closed-form} reduces $\Pi_s$ to three matrix factors. We now 
bound the magnitude $\|\Pi_s\|_{1 \to 1}$, which is used in the bounds on 
pieces (A), (C), (D), and the difference 
$\|\hat \Pi_s - \Pi_s\|_{1 \to 1}$ used in piece (B).

\begin{lemma}[Magnitude bound on $\Pi_s$]
\label{lem:Pi-magnitude}
Under the assumptions of Theorem~2,
\begin{equation}
\|\Pi_s\|_{1 \to 1} \leq \frac{|\mathcal{Z}| - 1}{b_{\min}}.
\end{equation}
The same bound holds for $\|\hat \Pi_s\|_{1 \to 1}$.
\end{lemma}

\begin{proof}
Recall the identity $\|M\|_{1 \to 1} = \max_j \sum_i |M_{ij}|$ (maximum 
column 1-norm). By Eq. \ref{eq:Pi-closed-form} and submultiplicativity of operator norms,
\begin{equation}
\|\Pi_s\|_{1 \to 1} \leq \|D_{t+k}\|_{1 \to 1} \cdot 
\|I - \mathbf{1}b_{t+k}^\top\|_{1 \to 1} \cdot \|D_{s+1}^{-1}\|_{1 \to 1}.
\end{equation}
We bound each factor.

\emph{Bound on $\|D_{t+k}\|_{1 \to 1}$.} Each column of $D_{t+k}$ has a 
single nonzero entry $b_{t+k}^j \in [0, 1]$, so
\begin{equation}
\|D_{t+k}\|_{1 \to 1} = \max_j b_{t+k}^j \leq 1.
\end{equation}

\emph{Bound on $\|D_{s+1}^{-1}\|_{1 \to 1}$.} Each column of $D_{s+1}^{-1}$ 
has a single nonzero entry $1/b_{s+1}^j$.
\begin{equation}
\|D_{s+1}^{-1}\|_{1 \to 1} = \max_j \frac{1}{b_{s+1}^j} 
= \frac{1}{\min_j b_{s+1}^j} \leq \frac{1}{b_{\min}}.
\end{equation}

\emph{Bound on $\|I - \mathbf{1}b_{t+k}^\top\|_{1 \to 1}$.} Column $j$  in $\|I - \mathbf{1}b_{t+k}^\top\|_{1 \to 1}$ has entry $(1 - b_{t+k}^j)$ at position $j$ and $-b_{t+k}^j$ at each of the other $|\mathcal{Z}| - 1$ positions. Its 1-norm is
\begin{equation}
(1 - b_{t+k}^j) + (|\mathcal{Z}| - 1)\,b_{t+k}^j 
= 1 + (|\mathcal{Z}| - 2)\,b_{t+k}^j 
\leq |\mathcal{Z}| - 1.
\end{equation}

\emph{Combining,}
\begin{equation}
\|\Pi_s\|_{1 \to 1} \leq 1 \cdot (|\mathcal{Z}| - 1) \cdot \frac{1}{b_{\min}} 
= \frac{|\mathcal{Z}| - 1}{b_{\min}}.
\end{equation}
The same bound applies to $\|\hat \Pi_s\|_{1 \to 1}$ via the analogous 
argument along the approximate trajectory.
\end{proof}

We now bound the difference $\|\hat\Pi_s - \Pi_s\|_{1 \to 1}$, which feeds 
directly into piece (B).

\begin{lemma}[Perturbation bound on $\Pi_s$]
\label{lem:Pi-difference}
Under the assumptions of Theorem \ref{thm:grad-error},
\begin{equation}
\|\hat\Pi_s - \Pi_s\|_{1 \to 1} \leq \frac{ (6|\mathcal{Z}| - 4) k\, \varepsilon}{b_{\min}^2},
\end{equation}
\end{lemma}

\begin{proof}
By the closed forms in Eq.\ref{eq:Pi-closed-form} and add-subtract telescoping along the three 
factors,
\begin{align}
\hat \Pi_s - \Pi_s 
&= \underbrace{(\hat D_{t+k} - D_{t+k})(I - \mathbf{1}\hat b_{t+k}^\top)\hat D_{s+1}^{-1}}_{\text{(i)}} \notag \\
&\quad + \underbrace{D_{t+k}\,\mathbf{1}(b_{t+k} - \hat b_{t+k})^\top\,\hat D_{s+1}^{-1}}_{\text{(ii)}} \notag \\
&\quad + \underbrace{D_{t+k}(I - \mathbf{1}b_{t+k}^\top)(\hat D_{s+1}^{-1} - D_{s+1}^{-1})}_{\text{(iii)}},
\label{eq:Pi-three-pieces}
\end{align}
where in (ii) the $I$'s cancel and the rank-one difference $(I - 
\mathbf{1}\hat b_{t+k}^\top) - (I - \mathbf{1}b_{t+k}^\top) = 
\mathbf{1}(b_{t+k} - \hat b_{t+k})^\top$ remains. We bound each piece in 
$\|\cdot\|_{1 \to 1}$ via submultiplicativity, using the magnitude bounds 
from the proof of Lemma~\ref{lem:Pi-magnitude} and Corollary \ref{cor:traj-drift}.

\emph{Piece (i).} The diagonal difference satisfies $\|\hat D_{t+k} - D_{t+k}\|_{1 \to 1} 
= \max_z |\hat b_{t+k}^z - b_{t+k}^z| \leq \|\hat b_{t+k} - b_{t+k}\|_1 \leq 2k\varepsilon$. 
Combined with $\|I - \mathbf{1}\hat b_{t+k}^\top\|_{1 \to 1} \leq |\mathcal{Z}| - 1$ and 
$\|\hat D_{s+1}^{-1}\|_{1 \to 1} \leq 1/b_{\min}$,
\begin{equation}
\|\text{(i)}\|_{1 \to 1} \leq 2k\varepsilon \cdot (|\mathcal{Z}| - 1) \cdot \frac{1}{b_{\min}} 
= \frac{2(|\mathcal{Z}| - 1)\, k \varepsilon}{b_{\min}}.
\end{equation}

\emph{Piece (ii).} The rank-one matrix $\mathbf{1}(b_{t+k} - \hat b_{t+k})^\top$ 
has column $j$ equal to $(b_{t+k}^j - \hat b_{t+k}^j)\mathbf{1}$, whose 1-norm is 
$|\mathcal{Z}| \cdot |b_{t+k}^j - \hat b_{t+k}^j|$. Hence $\|\mathbf{1}(b_{t+k} 
- \hat b_{t+k})^\top\|_{1 \to 1} \leq |\mathcal{Z}|\,\|b_{t+k} - \hat b_{t+k}\|_\infty 
\leq 2|\mathcal{Z}| k\varepsilon$. Combined with $\|D_{t+k}\|_{1 \to 1} \leq 1$ and 
$\|\hat D_{s+1}^{-1}\|_{1 \to 1} \leq 1/b_{\min}$,
\begin{equation}
\|\text{(ii)}\|_{1 \to 1} \leq 1 \cdot 2|\mathcal{Z}| k\varepsilon \cdot \frac{1}{b_{\min}} 
= \frac{2|\mathcal{Z}|\, k \varepsilon}{b_{\min}}.
\end{equation}

\emph{Piece (iii).} The diagonal difference $\hat D_{s+1}^{-1} - D_{s+1}^{-1}$ 
has entries
\begin{equation}
\frac{1}{\hat b_{s+1}^z} - \frac{1}{b_{s+1}^z} = 
\frac{b_{s+1}^z - \hat b_{s+1}^z}{b_{s+1}^z \hat b_{s+1}^z}.
\end{equation}
Using $b_{s+1}^z, \hat b_{s+1}^z \geq b_{\min}$ from Assumption~(3),
\begin{equation}
\|\hat D_{s+1}^{-1} - D_{s+1}^{-1}\|_{1 \to 1} 
\leq \max_z \frac{|b_{s+1}^z - \hat b_{s+1}^z|}{b_{s+1}^z \hat b_{s+1}^z} 
\leq \frac{2k\varepsilon}{b_{\min}^2}.
\end{equation}
Combined with $\|D_{t+k}\|_{1 \to 1} \leq 1$ and $\|I - \mathbf{1}b_{t+k}^\top\|_{1 \to 1} 
\leq |\mathcal{Z}| - 1$,
\begin{equation}
\|\text{(iii)}\|_{1 \to 1} \leq 1 \cdot (|\mathcal{Z}| - 1) \cdot \frac{2k\varepsilon}{b_{\min}^2} 
= \frac{2(|\mathcal{Z}| - 1)\, k \varepsilon}{b_{\min}^2}.
\end{equation}

\emph{Combining.} By the triangle inequality applied to 
Eq.~\ref{eq:Pi-three-pieces}, and using $1/b_{\min} \leq 1/b_{\min}^2$ since 
$b_{\min} \leq 1$,
\begin{align}
\|\hat \Pi_s - \Pi_s\|_{1 \to 1} 
&\leq \frac{2(|\mathcal{Z}| - 1)\, k\varepsilon}{b_{\min}} 
+ \frac{2|\mathcal{Z}|\, k\varepsilon}{b_{\min}} 
+ \frac{2(|\mathcal{Z}| - 1)\, k\varepsilon}{b_{\min}^2} \notag \\
&\leq \frac{(6|\mathcal{Z}| - 4)\, k\varepsilon}{b_{\min}^2}.
\end{align}
\end{proof}

\noindent We can now bound piece (B). By submultiplicativity,
\begin{equation}
\|(B)_s\|_2 \leq \|g\|_\infty \cdot \|\hat\Pi_s - \Pi_s\|_{1 \to 1} \cdot 
\|\hat S_s\|_{\infty \to 1} \cdot \|\hat G_s\|_{2 \to \infty}.
\end{equation}
Combining Lemma~\ref{lem:Pi-difference} with the magnitude bounds from 
Proposition \ref{prop:magnitude},
\begin{equation}
\label{eq:piece-B-final}
\|(B)_s\|_2 \leq L_V \cdot \frac{(6|\mathcal{Z}| - 4)\, k\varepsilon}{b_{\min}^2} 
\cdot 2 \cdot G_\pi 
= \frac{2(6|\mathcal{Z}| - 4)\, k \varepsilon\, L_V\, G_\pi}{b_{\min}^2}.
\end{equation}

\subsubsection{Assembly: combining the four pieces}
\label{sec:assembly}

We combine the bounds on pieces (A), (B), (C), (D), together with the bound 
$\|\Pi_s\|_{1 \to 1} \leq (|\mathcal{Z}| - 1)/b_{\min}$ from 
Lemma~\ref{lem:Pi-magnitude}, to produce the final bound stated in Theorem \ref{thm:grad-error}.

\paragraph{Per-summand bounds.}
Substituting Lemma~\ref{lem:Pi-magnitude} into Eqs.~80, 83, and 87:
\begin{align}
\|(D)_s\|_2 &\leq \frac{2(|\mathcal{Z}| - 1) L_V L_\pi \|\hat o_s - o_s\|_\infty}{b_{\min}}, \\
\|(A)_s\|_2 &\leq \frac{4(|\mathcal{Z}| - 1) k\varepsilon L_g G_\pi}{b_{\min}}, \\
\|(C)_s\|_2 &\leq \frac{6(|\mathcal{Z}| - 1) k\varepsilon L_V G_\pi}{b_{\min}}, \\
\|(B)_s\|_2 &\leq \frac{2(6|\mathcal{Z}| - 4) k\varepsilon L_V G_\pi}{b_{\min}^2} \quad \text{(\ref{lem:Pi-difference})}.
\end{align}

\paragraph{Sum across the window.}
Summing each piece across $s = t, \ldots, t+k-1$ (with $k$ summands), and 
using $\|\hat o_s - o_s\|_\infty \leq \sup_t \|\hat o_t - o_t\|_\infty$:
\begin{align}
\sum_s \|(D)_s\|_2 &\leq \frac{2(|\mathcal{Z}| - 1) k L_V L_\pi \sup_t \|\hat o_t - o_t\|_\infty}{b_{\min}}, \\
\sum_s \|(A)_s\|_2 &\leq \frac{4(|\mathcal{Z}| - 1) k^2\varepsilon L_g G_\pi}{b_{\min}}, \\
\sum_s \|(C)_s\|_2 &\leq \frac{6(|\mathcal{Z}| - 1) k^2\varepsilon L_V G_\pi}{b_{\min}}, \\
\sum_s \|(B)_s\|_2 &\leq \frac{2(6|\mathcal{Z}| - 4) k^2\varepsilon L_V G_\pi}{b_{\min}^2}.
\end{align}

\paragraph{Final assembly.}
By Eq.~74, $\|\nabla_\theta L_{\mathrm{approx}} - \nabla_\theta L_{\mathrm{true}}\|_2$ 
is bounded by the sum of the four. Using $1/b_{\min} \leq 1/b_{\min}^2$ 
(since $b_{\min} \leq 1$) to put all terms on a common denominator, and 
$\max(L_V, L_g)$ to absorb both Lipschitz constants, the $G_\pi$ contributions 
from (A), (B), (C) combine as
\begin{equation}
\frac{\big[4(|\mathcal{Z}|-1) + 6(|\mathcal{Z}|-1) + 2(6|\mathcal{Z}|-4)\big]\, k^2\varepsilon \, \max(L_V, L_g)\, G_\pi}{b_{\min}^2} 
= \frac{(22|\mathcal{Z}| - 18)\, k^2\varepsilon \, \max(L_V, L_g)\, G_\pi}{b_{\min}^2},
\end{equation}
while the $L_\pi$ contribution from (D) is bounded by 
$\frac{2(|\mathcal{Z}|-1) k L_V L_\pi \sup_t \|\hat o_t - o_t\|_\infty}{b_{\min}^2}$.

Combining,
\begin{equation}
\big\|\nabla_\theta L_{\mathrm{approx}} - \nabla_\theta L_{\mathrm{true}}\big\|_2 
\leq \frac{(22|\mathcal{Z}| - 18) \cdot k \cdot \max(L_V, L_g)}{b_{\min}^2}
\Big[k\varepsilon G_\pi + L_\pi \sup_t \|\hat o_t - o_t\|_\infty\Big].
\end{equation}
This completes the proof of Theorem \ref{thm:grad-error}.
\end{proof}

\subsection{Additional Training and Diagnostic Figures}
\label{app:diagnostic-figures}

The main Avalon comparison in Figure~\ref{fig:avalon-evals} uses a 3-seed evaluation shared by PPO, BBM, and D-BOS ($42,43,44$). Estimated-observation Avalon runs are also included in the graphs. We noticed a performance boost with using the estimated-observation mode as our agent is now able to shape multiple agents at the same time. 

\begin{figure}[htbp]
    \centering
    \includegraphics[width=\linewidth]{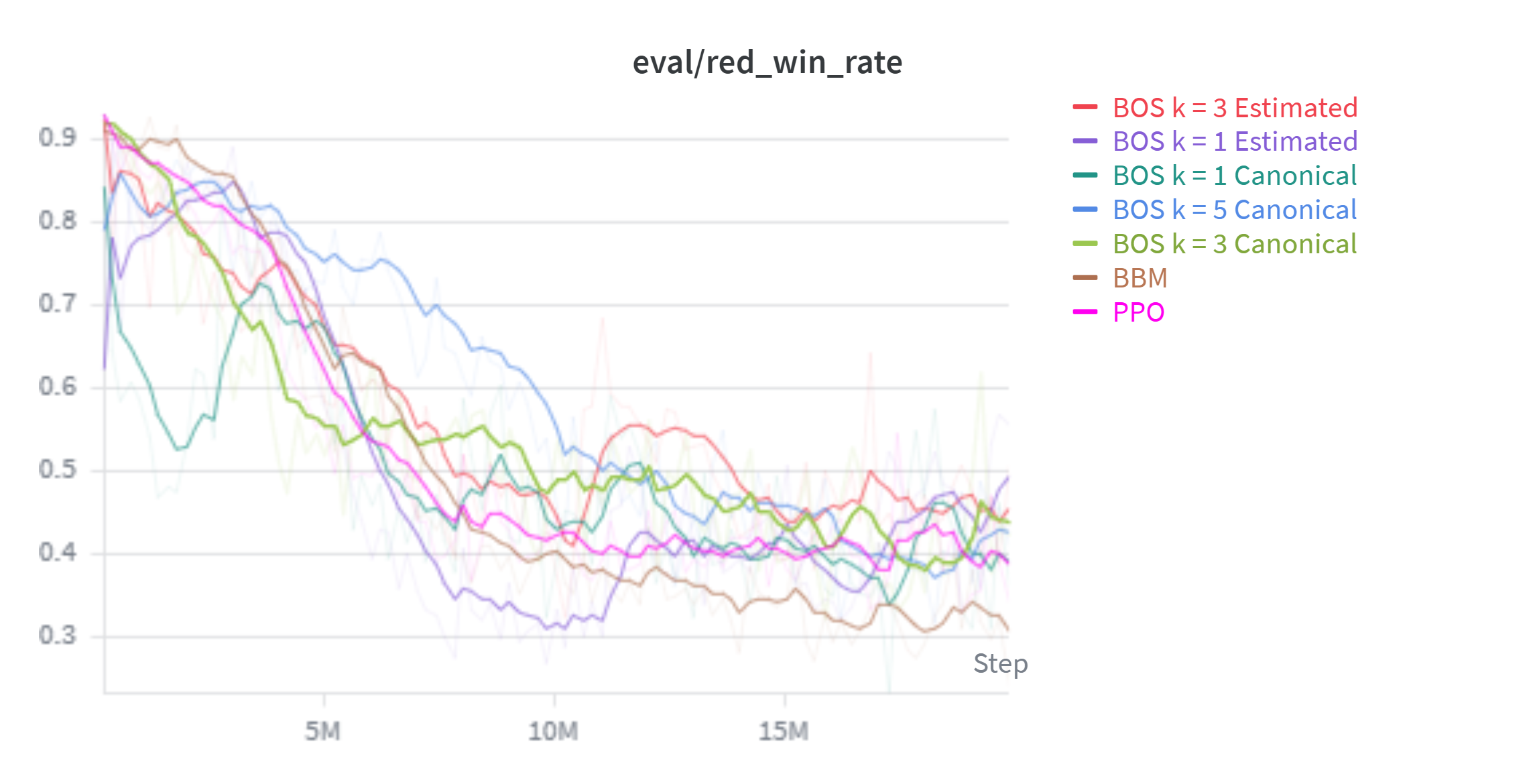}
    \caption{Avalon Evaluation metrics over time. The estimated runs dip much earlier compared to the other methods, but regain performance over time to end with a higher win rate, suggesting that shaping multiple observers probably sacrifices performance earlier in the runs for better rewards later.}
    \label{fig:avalon-evals}
\end{figure}

We further show in Table \ref{tab:avalon-per-seed}, how the Avalon Estimated-observations runs perform per seed. We notice that the estimated observation mode consistently outperforms baseline showcasing its effectiveness in maintaining a better win rate than baselines.

\begin{table}[htbp]
    \centering
    \small
    \begin{tabular}{lcccc}
        \toprule
        & \multicolumn{3}{c}{Avalon Win Rate (Per Independent Seed)} & \\
        \cmidrule(lr){2-4}
        Method & Seed 42 & Seed 43 & Seed 44 & Aggregate (Mean $\pm$ SE) \\
        \midrule
        PPO (No shaping) & 0.340 & 0.290 & 0.400 & 0.343 $\pm$ 0.032 \\
        BBM & 0.440 & 0.150 & 0.140 & 0.243 $\pm$ 0.098 \\
        \midrule
        D-BOS, $k=1$ Est. (Ours) & \textbf{0.490} & \textbf{0.530} & \textbf{0.650} & \textbf{0.557 $\pm$ 0.048} \\
        D-BOS, $k=3$ Est. (Ours) & \textbf{0.53} & \textbf{0.49} & \textbf{0.5} & \textbf{0.507 $\pm$ 0.012} \\
        \bottomrule
    \end{tabular}
    \vspace{3pt}
    \caption{Per-seed evaluation win rates in Avalon. While hidden-role environments naturally exhibit high inter-seed variance, breaking down the performance by independent runs demonstrates that D-BOS consistently outperforms baselines across initializations. (Note: D-BOS $k=1$ Estimated significantly dominates on Seed 44).}
    \label{tab:avalon-per-seed}
\end{table}

Figure ~\ref{fig:rtg-evals} showcases the training graphs for the Rescue-the-General environment for multiple different methods. We can see that D-BOS sustains performance for a longer period of time compared to baselines. All methods eventually drop in performance due to the environment dynamics, where blue ends up winning. We also notice a drop in performance proportional to the value of $k$ as suggested by Theorem \ref{thm:grad-error}.

\begin{figure}[htbp]
    \centering
    % --- Row 1 ---
    \begin{minipage}{0.48\linewidth}
        \centering
        \includegraphics[width=\linewidth]{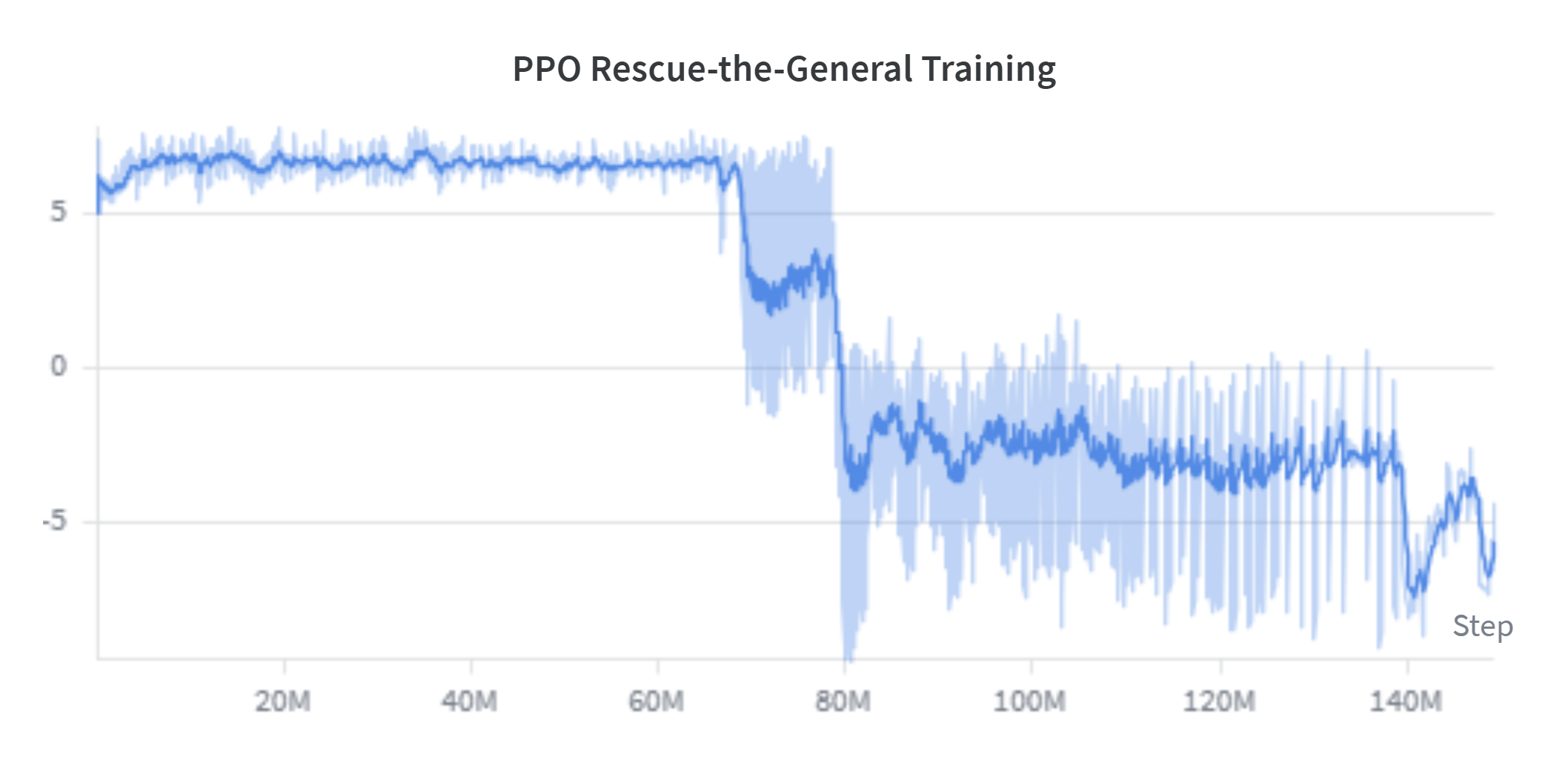}
        \small (a) PPO (No shaping)
    \end{minipage}
    \hfill
    \begin{minipage}{0.48\linewidth}
        \centering
        \includegraphics[width=\linewidth]{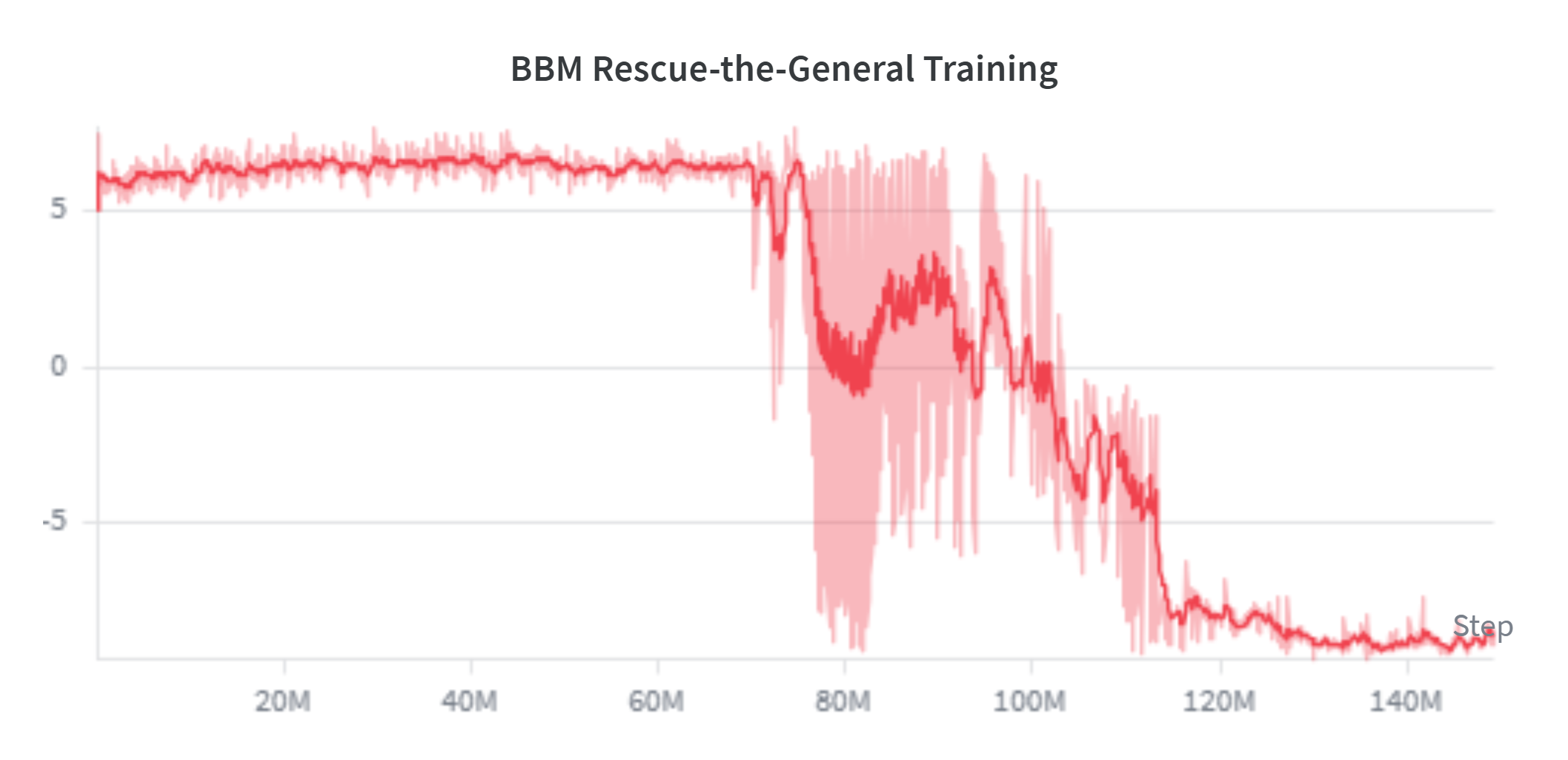}
        \small (b) BBM
    \end{minipage}

    \vspace{1em}

    % --- Row 2 ---
    \begin{minipage}{0.48\linewidth}
        \centering
        \includegraphics[width=\linewidth]{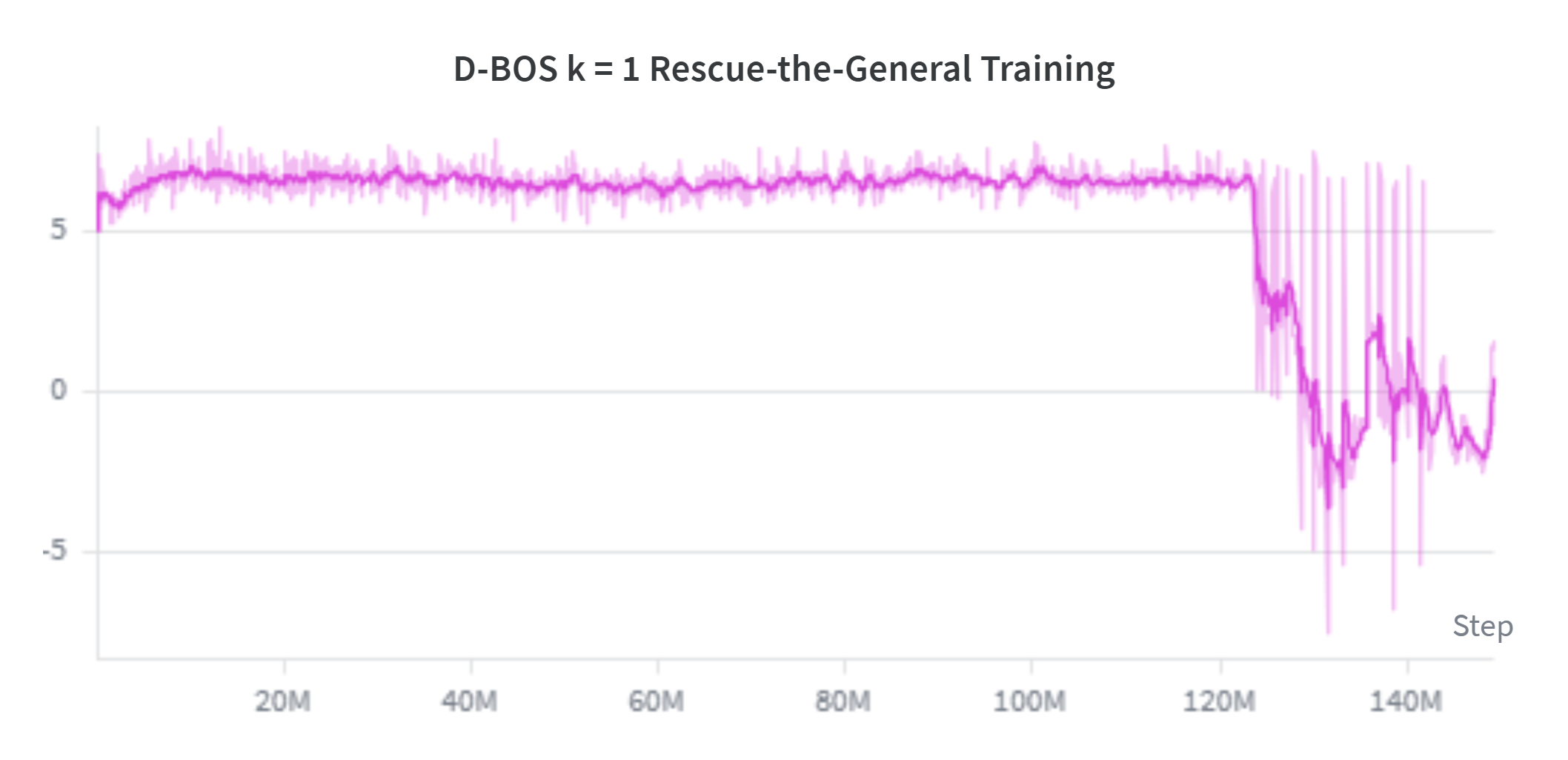}
        \small (c) D-BOS $k=1$
    \end{minipage}
    \hfill
    \begin{minipage}{0.48\linewidth}
        \centering
        \includegraphics[width=\linewidth]{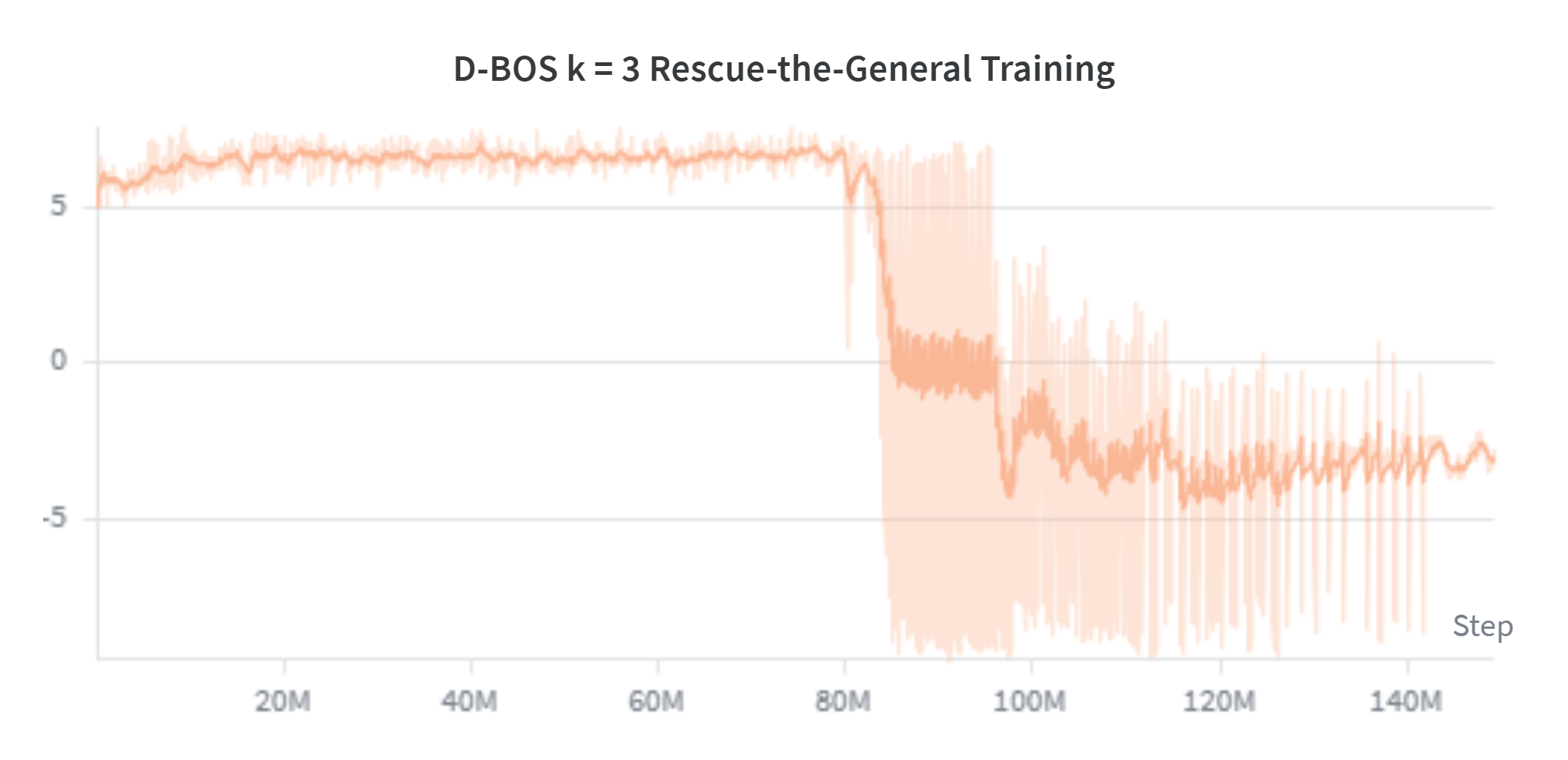}
        \small (d) D-BOS $k=3$
    \end{minipage}

    \vspace{1em}

    % --- Row 3 (Centered) ---
    \begin{minipage}{0.48\linewidth}
        \centering
        \includegraphics[width=\linewidth]{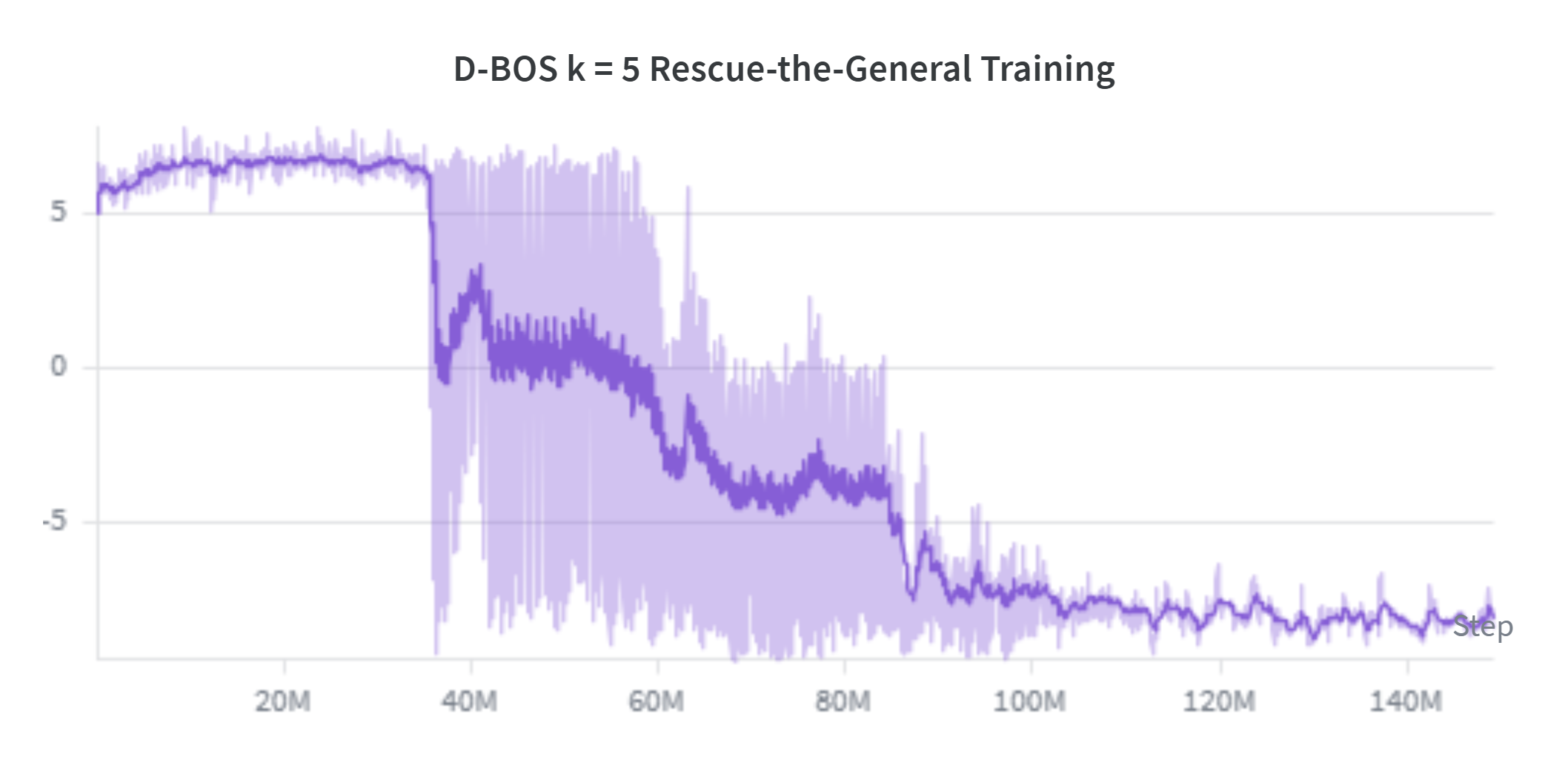}
        \small (e) D-BOS $k=5$
    \end{minipage}

    \vspace{1em}
    \caption{Rescue-the-General Training graphs over time. We observe that D-BOS with $k=1$ sustains performance significantly longer than baseline methods before the characteristic performance drop at 120M steps.}
    \label{fig:rtg-evals}
\end{figure}

\begin{figure}
    \centering
    \includegraphics[width=0.8\linewidth]{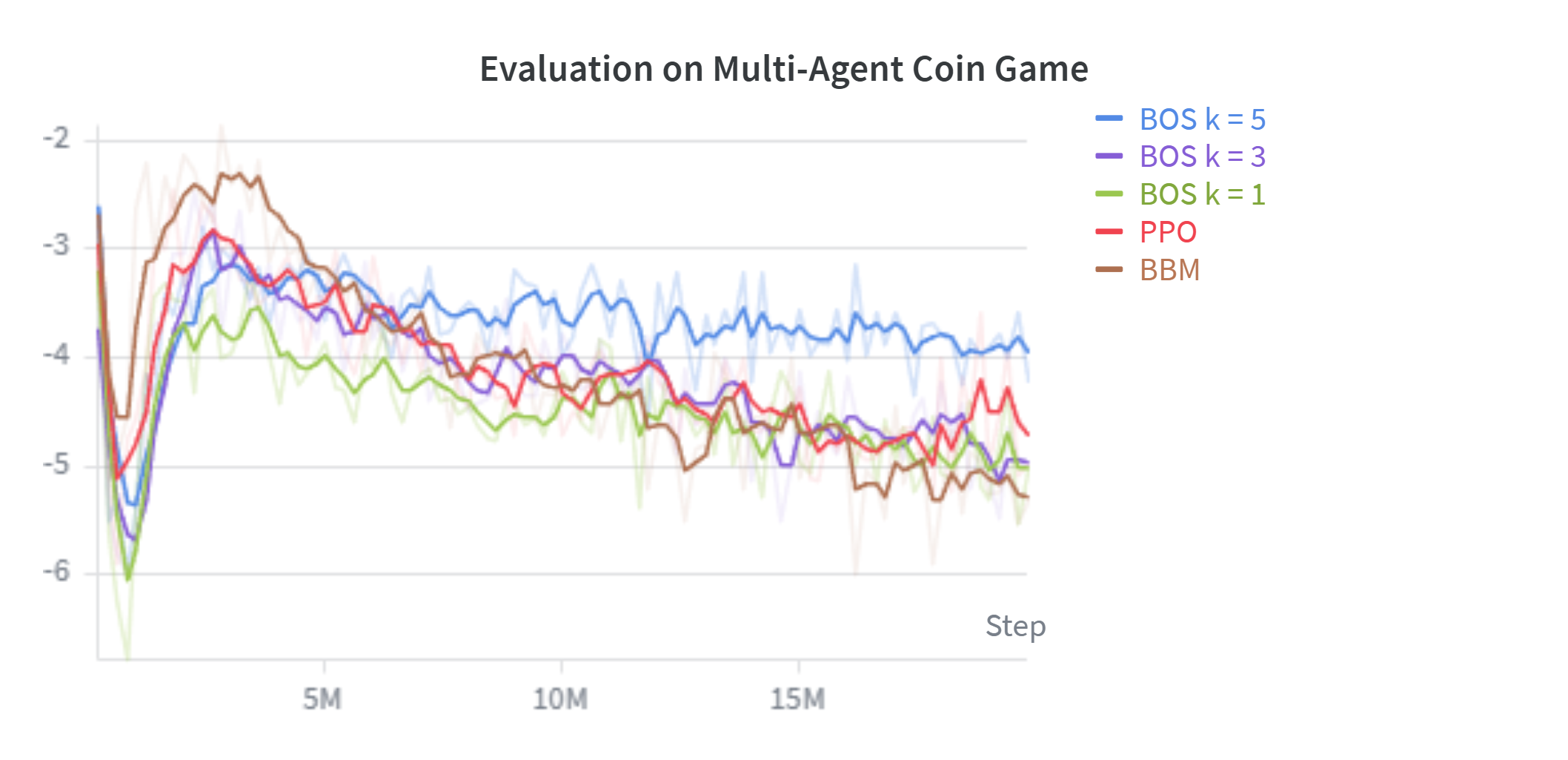}
    \caption{This figure shows the graphs for the Multi-Agent Coin Game evaluation across 3 seeds. We can see that in this method, it becomes quite hard for D-BOS to shape beliefs when the role is changing every episode. However, this serves as a stress test, and it motivates future work as D-BOS shows potential in such environments}
    \label{fig:placeholder}
\end{figure}

\subsection{Compute and Hyperparameters}
\label{app:compute-hparams}

All methods within an environment use the same rollout budget, optimizer family, policy architecture, and seed set.  RTG uses the CNN policy path with 64 parallel environments, rollout length 16, PPO learning rate $2.5\times 10^{-4}$, discount $\gamma=0.99$, GAE $\lambda=0.95$, and checkpoints/evaluation points every $10$M environment steps. The main RTG runs use $150$ million environment steps per seed. Avalon5 and CoinGame use the MLP policy path with 16 parallel environments, rollout length 32, two PPO epochs, two minibatches, learning rate $5\times10^{-4}$, discount $\gamma=0.99$, GAE $\lambda=0.95$, and a 20M-step training budget. Avalon uses a hidden dimension 128 and an entropy coefficient of 0.02; CoinGame uses hidden dimension of 64 and an entropy coefficient of 0.01.

We ran the experiments on A100-40GB, A100-80GB, and A30 GPUs. RTG was the most expensive setting, taking roughly two days per method/seed at the final budget. Avalon and CoinGame runs were substantially lighter, typically completing overnight for each method/seed, depending on the queue and GPU type. The dominant compute difference is across environments rather than across PPO, BBM, and D-BOS within the same environment.

\paragraph{Shaping-strength settings.}
The scalar shaping coefficient is not directly comparable across BBM and D-BOS. BBM scales an auxiliary belief-manipulation reward, while D-BOS scales a belief-gradient correction. We therefore select the shaping strength separately for each shaping mechanism while keeping the environment, architecture, optimizer, rollout budget, and seeds fixed. In Avalon, BBM uses $\lambda = 0.5$, matching the setting reported in the BBM reference paper~\citep{bbm} and the best stable value from our quick validation runs. The final Avalon D-BOS runs use $\lambda=1$, which gave the most reliable behavior in our tuning runs. In CoinGame, both BBM and D-BOS use $\lambda=0.5$; in RTG, D-BOS and BBM both use $\lambda=0.5$. These coefficients should be interpreted as method-specific tuning parameters rather than matched parameters

\subsection{Broader Impact and Societal Considerations}
\label{app:checklist-}
D-BOS provides agents with the ability to influence what others believe. While we study this capability in controlled game environments, the underlying mechanism of optimizing actions to shape observer posteriors has a dual-use potential. On the positive side, belief-aware agents can improve coordination in human-robot teaming, make autonomous systems more legible to human partners, and enable more effective communication under information asymmetry. On the negative side, the same mechanism could in principle be applied to train agents that systematically deceive human users or manipulate beliefs in adversarial settings. We note that D-BOS operates within a fixed game structure with known role sets and does not generalize to open-ended deception in natural language or unconstrained environments. Nevertheless, as belief-shaping methods mature, careful attention to deployment contexts and the development of counter-deception methods will be important. While we do test in an adversarial setting, our goal is to eventually use our method for better cooperation amongst humans and robots and do not condone the misuse of belief-based opponent shaping.

\subsection{LLM Usage}
\label{app:checklist-llm-usage}
We used large language models as a writing and editing aid during manuscript preparation, including assistance with prose drafting, LaTeX formatting, and proof presentation and formatting. All technical content, mathematical derivations, experimental design, implementation, and scientific claims are the original work of the authors. We referred to some help from LLMs in implementing and refining our codebase, but LLMs were not an important part of the research, nor did they contribute any technical details integral to our methods.

% \end{proof}

%%%%%%%%%%%%%%%%%%%%%%%%%%%%%%%%%%%%%%%%%%%%%%%%%%%%%%%%%%%%

\newpage

\end{document}